\documentclass[11pt]{article}

\usepackage[final]{acl}

\usepackage{times}
\usepackage{latexsym}

\usepackage[T1]{fontenc}

\usepackage[utf8]{inputenc}

\usepackage{microtype}

\usepackage{inconsolata}

\usepackage{graphicx}

\usepackage{hyperref}
\usepackage{amsmath}
\usepackage{amssymb}
\usepackage{mathtools}
\usepackage{amsthm}
\usepackage{tikz}
\usepackage{bm}
\usepackage[most]{tcolorbox}
\tcbuselibrary{breakable}
\usepackage{physics}
\usepackage[capitalize,noabbrev]{cleveref}
\usepackage{multirow}
\usepackage[table,xcdraw]{xcolor}
\usepackage[normalem]{ulem}
\useunder{\uline}{\ul}{}
\usepackage{dsfont}

\newtcolorbox{promptbox}[1]{%
    breakable,
    enhanced,
    colback=gray!4,
    colframe=gray!55!black,
    boxrule=0.45pt,
    arc=2.5pt,
    left=6pt, right=6pt, top=5pt, bottom=5pt,
    title=\textsf{\small\bfseries #1},
    fonttitle=\sffamily\small\bfseries,
    coltitle=white,
    colbacktitle=gray!55!black,
    attach boxed title to top left={yshift=-2mm, xshift=4mm},
    boxed title style={colback=gray!55!black, sharp corners, boxrule=0pt},
    fontupper=\small\ttfamily,
    before upper={\setlength{\parindent}{0pt}\setlength{\parskip}{2pt}},
}

\newtcolorbox{examplebox}[1]{%
    breakable,
    enhanced,
    colback=blue!2,
    colframe=blue!35!black,
    boxrule=0.45pt,
    arc=2.5pt,
    left=6pt, right=6pt, top=5pt, bottom=5pt,
    title=\textsf{\small\bfseries #1},
    fonttitle=\sffamily\small\bfseries,
    coltitle=white,
    colbacktitle=blue!35!black,
    attach boxed title to top left={yshift=-2mm, xshift=4mm},
    boxed title style={colback=blue!35!black, sharp corners, boxrule=0pt},
    fontupper=\small,
    before upper={\setlength{\parindent}{0pt}\setlength{\parskip}{2pt}},
}

\theoremstyle{plain}
\newtheorem{theorem}{Theorem}[section]

\newtheorem{lemma}[theorem]{Lemma}
\newtheorem{corollary}[theorem]{Corollary}
\theoremstyle{definition}

\theoremstyle{remark}

\title{Statistical Priors for Implicit Preferences: \\Decoupling Skill Selection as a Local Harness in Personal Agents}

\author{Zeyu Gan \quad Huayi Tang \quad Yong Liu\thanks{Corresponding Author.} \\
  Gaoling School of Artificial Intelligence \\
  Renmin University of China \\
  Beijing, China \\
  \texttt{\{zygan,huayitang,liuyonggsai\}@ruc.edu.cn} \\}

\begin{document}
\maketitle

\begin{abstract}
As Large Language Model (LLM) capabilities advance, locally deployed personal agents relying on API-based remote models and external skills have emerged as a novel paradigm. With the rapid expansion of available skills, enabling personal agents to learn and adapt to implicit user preferences becomes a critical challenge. However, local deployment constraints preclude complex centralized selection algorithms, creating an urgent need for a lightweight local preference harness. This paper explores the implementation of such a harness through a novel architecture that strictly decouples statistical preference learning from semantic intent parsing. Specifically, we leverage localized statistical results to influence and modulate the selection decisions of the remote LLM. Extensive evaluations demonstrate that our decoupled approach achieves the lowest cumulative regret and highest test accuracy, significantly outperforming traditional memory-augmented agents. We open-source our code at \url{https://github.com/ZyGan1999/Personalized-Skill-Selection}.

\end{abstract}

\section{Introduction}
\label{sec:introduction}


\begin{figure}[t]
    \centering
    \includegraphics[width=\linewidth]{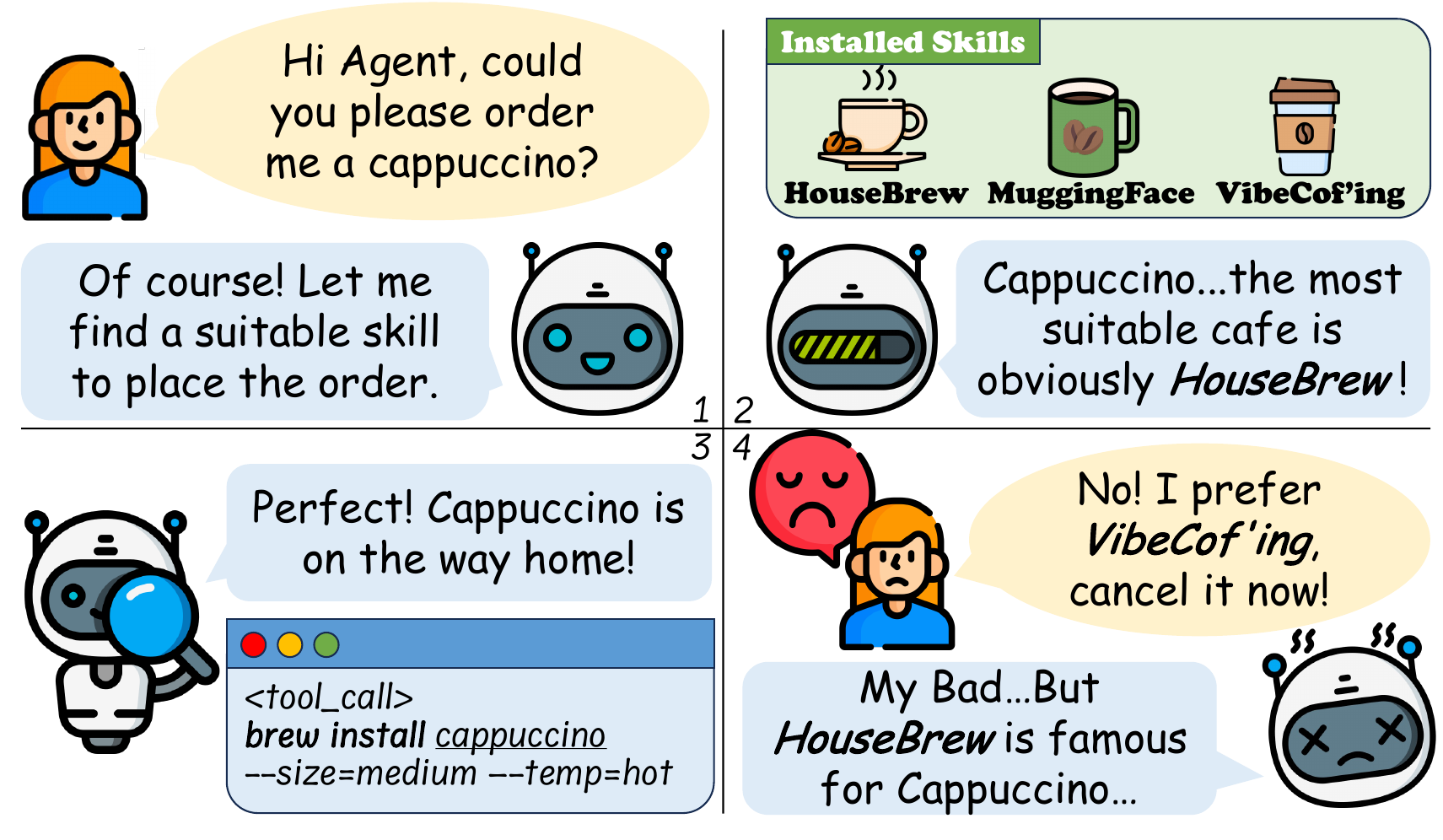}
    \caption{\textbf{A typical failure when personal agents are handling tasks with implicit user preference.} The agent selected a skill based on the user's instruction and its own knowledge, but ignored the user's implicit preferences, resulting in an unsatisfactory choice.}
    \label{fig:introduction}
    \vspace{-10pt}
\end{figure}

With the continuous advancement of Large Language Models (LLMs), the research paradigm has increasingly shifted from merely scaling model capabilities to transforming these powerful models into robust productivity tools. Against this backdrop, LLM-based agents have emerged as a prominent research frontier, giving rise to numerous commercial products capable of complex reasoning and task execution, such as Gemini's deep research~\citep{gemini-deep-research}, Manus~\citep{manus}, and Perplexity~\citep{perplexity}. Simultaneously, the growing demand for highly customizable, privacy-preserving, and deeply integrated digital assistants has catalyzed the rapid development of locally deployed personal agents (e.g., Claude Code~\citep{claude_code}, Codex~\citep{codex}, OpenClaw~\citep{openclaw}, HERMES Agent~\citep{hermes}, and Pi Agent~\citep{piagent}). These personal agents represent a novel interaction paradigm, intimately embedding themselves into users' daily workflows to provide highly individualized assistance. 

Distinct from traditional frameworks, personal agents heavily rely on API-based remote foundation models for complex semantic reasoning while operating within locally deployed execution environments. To extend their functional boundaries, a rapidly proliferating ecosystem of external skills and tools is being integrated into these local setups~\citep{qin2023toolllm,schick2023toolformer,patil2024gorilla}. As the number of available skills expands, dynamically selecting the most appropriate skill tailored to individual user preferences has emerged as a critical challenge~\citep{lin2025masstoolmultitasksearchbasedtool}. As illustrated in~\cref{fig:introduction}, when a user asks the agent to order a cappuccino, multiple functionally viable skills exist. The optimal choice is not determined by the explicit semantics of the user's query, but rather by their implicit, personal preference. Failing to accurately capture this preference inevitably leads to frustrating interactions. 

In real-world scenarios, user preferences for specific skills are rarely stated explicitly in every prompt; instead, they manifest through stochastic, high-frequency daily feedback and repetitive habits. For instance, a developer might consistently favor a specific linting skill for code inspection, or a user might habitually rely on a particular weather API. Consequently, the skill selection process in personal agents inherently encompasses two fundamentally different cognitive dimensions: \textbf{semantic understanding}, which parses explicit, ad-hoc intents from natural language instructions, and \textbf{statistical preference learning}, which identifies latent habits from historical interactions. 



However, effectively managing statistical preference learning poses a unique challenge for personal agents. Due to their localized deployment and strict privacy constraints, adopting computationally heavy or complex centralized recommendation algorithms is highly impractical, creating an urgent need for lightweight, on-device preference management. Currently, the prevailing solution relies on prompt-injected memory structures, forcing a single remote LLM to simultaneously handle both historical frequency tracking and semantic reasoning. This conflation leads to severe systemic failures. Beyond introducing high API latency and context window overflow, LLMs 
frequently get the logic lost in multi-turn conversations~\citep{laban2026llms}. Consequently, memory-based approaches fail to robustly capture fine-grained statistical priors and severely lack mathematical interpretability. 

To resolve this architectural bottleneck, we propose the \textsc{Local Harness}, a novel framework that enforces a strict physical and logical decoupling between statistical preference learning and semantic intent parsing. Specifically, we delegate the probabilistic credit-assignment problem of modeling implicit user habits to a locally deployed, highly efficient statistical primitive. This local module natively manages the exploration-exploitation tradeoff and serves as the primary decision-maker. Simultaneously, the high-latency remote LLM is entirely removed from the high-frequency execution critical path, strictly reserved as a semantic exception handler to process explicit lexical overrides. The main contributions of this paper are three-fold: 
(1) We identify the fundamental architectural flaw of conflating statistical preference learning with semantic reasoning in modern memory-augmented personal agents. 
(2) We propose the \textsc{Local Harness}, a lightweight, decoupled architecture that synergizes a local statistical estimator with a remote LLM exception handler to accurately model user preferences. 
(3) Since preference-driven skill selection is a newly identified problem with no existing simulation environment, we construct \textsc{ToolBench-60}, a dedicated sandbox, and conduct extensive empirical evaluations across diverse foundation models on it, demonstrating that our decoupled approach uniquely achieves both the lowest cumulative regret and highest test accuracy. 


The remainder of this paper is organized as follows. \Cref{sec:related-work} reviews related work on LLM agents and memory mechanisms. \Cref{sec:method} formalizes the preference-driven skill selection problem and details the \textsc{Local Harness} architecture. \Cref{sec:theoretical-analysis} presents theoretical analyses. \Cref{sec:experiment} discusses our extensive experimental setup and results. Finally, \cref{sec:conclusion} concludes the paper. 
\section{Related Work}
\label{sec:related-work}

\subsection{LLM Agents and Tool Use}
Empowering language models with external action execution has emerged as a cornerstone for building autonomous agentic systems~\citep{qin2023toolllm,yao2022react,zhou2026externalizationllmagentsunified,gan2026beyond}. Early methodologies established structured tool utilization by mapping textual intents to discrete API parameters~\citep{schick2023toolformer,patil2024gorilla}. To scale these capabilities, subsequent frameworks optimized execution paths across dense multi-tool environments and complex multi-turn planning loops~\citep{lin2025masstoolmultitasksearchbasedtool}. \citet{wang2026openclawrltrainagentsimply} further achieves interactive policy alignment by optimizing tool invocation through direct conversational feedback. While these tool-utilization paradigms excel at resolving explicit semantic tasks, they fundamentally formulate selection as an intent parsing problem, limiting the capability of capturing long-term user routines. 

\subsection{Harness Engineering}

The concept of designing specialized harnesses or orchestrators around LLMs has gained traction to mitigate the unreliability and high latency of raw model inference~\citep{harness-engineering}. A group of research emphasize formalizing the programmatic architecture surrounding LLMs into systematic harness engineering rather than relying on ad-hoc scaffolding~\citep{li2026agentharness,zhou2026externalizationllmagentsunified}. Subsequent frameworks have advanced these architectures by introducing structured execution runtimes~\citep{huang2026affordanceagentharnessverificationgated}, flexible natural language interfaces~\citep{pan2026naturallanguageagentharnesses}, and automated evolution mechanisms~\citep{lin2026agenticharnessengineeringobservabilitydriven}. More recently, the community has also begun exploring end-to-end optimization of model harnesses~\citep{lee2026metaharnessendtoendoptimizationmodel}. 
While these orchestration frameworks excel at structuring explicit logic and execution infrastructure, their heavy engineering footprints make them impractical for lightweight, locally deployed personal agents. 

\subsection{Memory-Augmented LLMs}
Recent advancements emphasize treating memory as a first-class cognitive primitive rather than a simple retrieval mechanism~\citep{hu2026memoryageaiagents}. 
Foundational frameworks such as Generative Agents~\citep{park2023generative} pioneered experiential memory by simulating human-like episodic recording. 
Further refining methods construct dynamic memory for robust agent systems~\citep {zhong2024memorybank,shinn2024reflexion}. 
\citet{zhang2026memgen} further achieves memory management in a generative manner. More recently, the community has also begun exploring self-evolving memory systems~\citep{pan2026mstartaskdeservesmemory}. A closely related line of work personalizes LLMs through natural-language \emph{personas} rather than raw logs. SynthesizeMe~\citep{ryan2025synthesizeme} induces persona-guided prompts from user interactions to build personalized reward models, and ACE~\citep{zhang2026agentic} maintains an incrementally revised ``playbook'' distilled from past episodes. While these prompt-injected memory systems excel at retrieving explicit semantic facts, they frequently suffer from context overflow with dense logs, underscoring the necessity of our decoupled local harness to handle implicit statistical preferences. 

\section{Method}
\label{sec:method}

\begin{figure}[t]
    \centering
    \includegraphics[width=\linewidth]{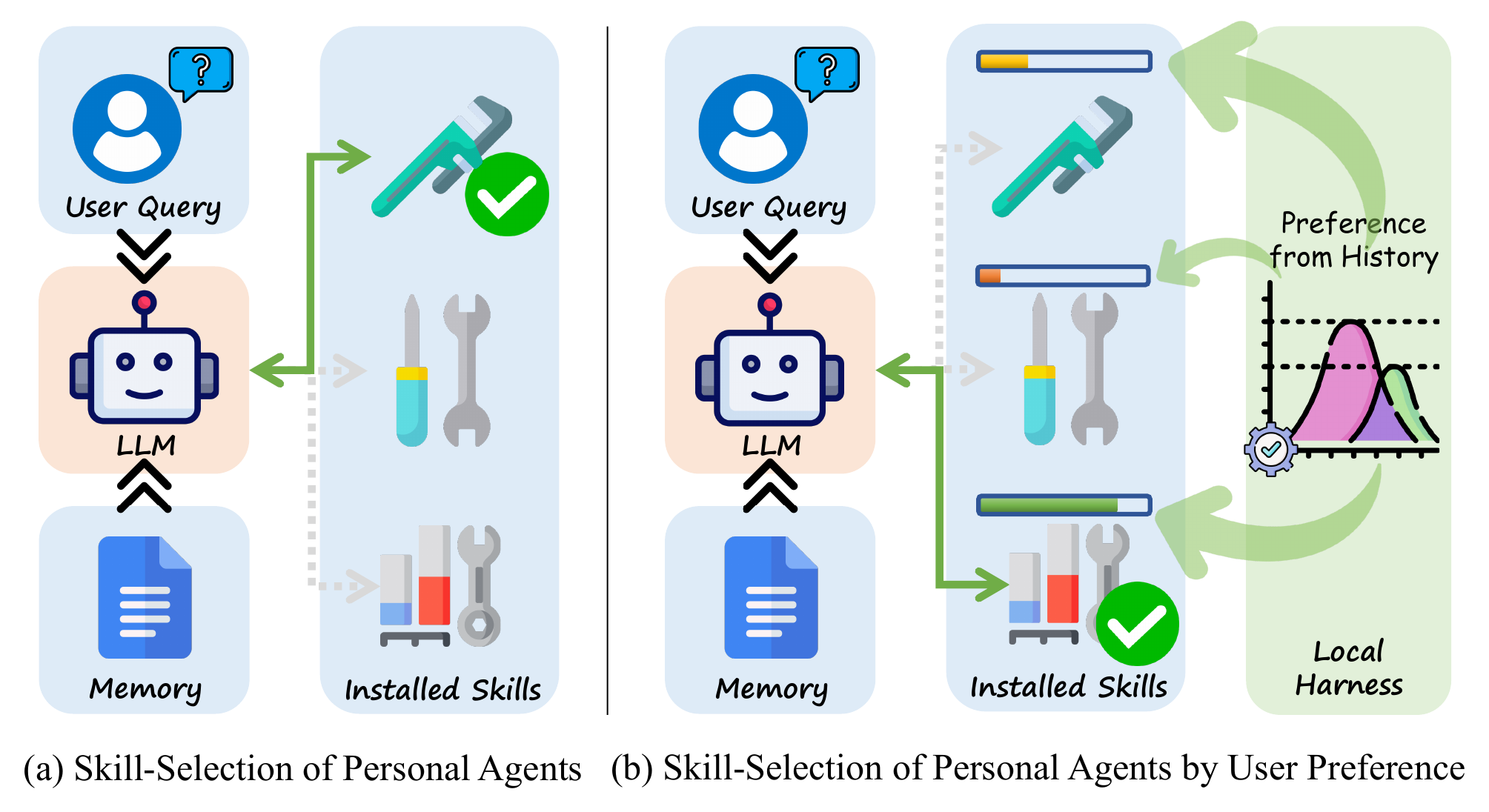}
    \caption{
    \textbf{Overview of the personalized skill selection problem and proposed architecture.} (a) A classic skill selection process based on user query.  (b) Our \textsc{Local Harness} architecture improves this by using a local statistical estimator for preference learning. 
    } 
    \label{fig:tool-selection-of-personal-agents}
    \vspace{-10pt}
\end{figure}

In this section, we formalize the personalized skill selection problem (\cref{ssec:benchmark-and-task-formulation}) and present the \textsc{Local Harness} architecture, detailing its statistical priors (\cref{ssec:local-harness}) and end-to-end decision procedure (\cref{ssec:tool-selection-via-statistical-priors}).

\subsection{Task Formulation}
\label{ssec:benchmark-and-task-formulation}

Standard skill-selection workflows typically treat the decision process as a direct semantic mapping from a natural language query to a specific skill, as illustrated in \cref{fig:tool-selection-of-personal-agents}(a). In this conventional paradigm, the text of the query is assumed to contain sufficient surface-level intent to uniquely identify the target skill. However, in real-world personal agent interactions driven by stochastic and highly repetitive daily habits, user queries are frequently skill-agnostic, rendering purely text-based intent parsing fundamentally insufficient. 

To address this, we formalize personalized skill selection as a sequential decision-making process under implicit user preference. At each interaction round $t$, a user $u$ issues a query $q_t$. A shared domain classifier first predicts the target domain $\hat{d}(q_t) \in \mathcal{D}$, which restricts the available choices to a candidate skill inventory $\mathcal{T}_d$ containing $K$ skills. Crucially, the ground-truth optimal skill $a_t^* \in \mathcal{T}_d$ is governed by the user's latent preference distribution $\pi_u: \mathcal{D} \rightarrow \Delta(\mathcal{T}_d)$. For a standard query, $a_t^*$ is sampled directly from $\pi_u(\hat{d})$. 
Upon selecting skill $a_t$, the system receives a binary reward $r_t = \mathds{1}\{a_t = a_t^*\}$ from user's subsequent feedback, which is considered an easily accessible reward signal~\citep{wang2026openclawrltrainagentsimply}. The primary objective is to maximize the cumulative expected reward $\sum_{t=1}^T r_t$ over a horizon of $T$ rounds by implicitly learning $\pi_u$ from historical feedback while maintaining robust adaptation to explicit semantic exceptions.

\subsection{Decouple Skill Selection as Local Harness}
\label{ssec:local-harness}

To resolve the fundamental conflict of forcing a single remote LLM to simultaneously manage statistical preference optimization and semantic parsing, we introduce the \textsc{Local Harness} architecture, illustrated in \cref{fig:tool-selection-of-personal-agents}(b). 
Concretely, a locally running, lightweight statistical component operates as the primary, default decision-maker. This component efficiently maintains the state space of the user's stochastic preferences locally and handles the exploration-exploitation trade-off mathematically, which we call a \textit{statistical prior}. 
By delegating historical credit assignment to the local harness, and reserving the remote LLM for lexical overrides, the architecture achieves a robust balance between statistical personalization efficiency and zero-shot reasoning capabilities. 
We introduce the following two statistical priors implemented in this paper.

\paragraph{Frequency Prior.} The simplest realization of the \textsc{Local Harness} is a per-user empirical success-rate table. For every triple $(u, d, k)$ comprising a user $u$, a domain $d$, and a candidate skill $k \in \mathcal{T}_d$, this component maintains two scalars after $t$ rounds: the number of attempts $N_{u,d,k}$ and the number of successes $S_{u,d,k}$ (the count of rounds in which skill $k$ was selected and received a positive reward). Its preference distribution over the inferred domain's candidate set $\mathcal{T}_{\hat{d}(q_t)}$ is the normalized success rate:
\begin{equation}
    \hat{p}^{\mathrm{freq}}_{u,d}(k) \;=\; \frac{S_{u,d,k} / \max(N_{u,d,k},\, 1)}{\sum_{k' \in \mathcal{T}_d} S_{u,d,k'} / \max(N_{u,d,k'},\, 1)},
\label{eq:freq_prior}
\end{equation}
and its default action is the corresponding $\arg\max$. 
The frequency prior makes no use of the query text and performs no explicit exploration. It therefore characterizes the elementary local statistical primitive, and serves as a controlled point of comparison for richer harnesses. 

\paragraph{Bandit Prior.}
A more principled realization casts each $(u, d, k)$ as a separate contextual-bandit arm. We instantiate this with \textsc{LinUCB}~\citep{li2010contextual}: every arm maintains a $D$-dimensional parameter $\boldsymbol{\theta}_{u,d,k}$ together with a regularized Gram matrix $\mathbf{A}_{u,d,k}\in\mathbb{R}^{D\times D}$ and a moment vector $\mathbf{b}_{u,d,k}\in\mathbb{R}^{D}$, updated online for each $(u,d,k)$ as: 
\begin{equation}
\begin{aligned}
&\mathbf{A} \leftarrow \mathbf{A} + \boldsymbol{\phi}(q,k,d)\boldsymbol{\phi}(q,k,d)^\top,
\\
&\mathbf{b} \leftarrow \mathbf{b} + r\,\boldsymbol{\phi}(q,k,d),
\\
&\boldsymbol{\theta} = \mathbf{A}^{-1}\mathbf{b}. 
\end{aligned}
\nonumber
\label{eq:linucb_update}
\end{equation}
Here $r\in\{0,1\}$ represents the binary reward. 
The context vector $\boldsymbol{\phi}(q,k,d)\in\mathbb{R}^{D}$ is constructed by feature-hashing~\citep{weinberger2009feature} the query, the skill name, and the inferred domain, producing a deterministic, vocabulary-free representation that generalizes across queries with related surface forms. 
At decision time the harness computes an upper-confidence-bound (UCB) score: 
\begin{equation}
\begin{aligned}
&s_{u,d,k}(q) \;=\;\boldsymbol{\theta}_{u,d,k}^{\top}\boldsymbol{\phi}(q,k,d)\;
\\&+\;\alpha_{\mathrm{ucb}}\,\sqrt{\boldsymbol{\phi}(q,k,d)^\top \mathbf{A}_{u,d,k}^{-1} \boldsymbol{\phi}(q,k,d)}.
\end{aligned}
\label{eq:ucb_score}
\end{equation}
The default action is the corresponding $\arg\max_{k} s_{u,d,k}(q)$. Additionally, the UCB scores can be exposed as a tempered-softmax preference distribution over $\mathcal{T}_{\hat{d}(q_t)}$:
\begin{equation}
    \hat{p}^{\mathrm{bandit}}_{u,d}(k) \;=\; \frac{\exp\bigl(\beta\, s_{u,d,k}(q)\bigr)}{\sum_{k' \in \mathcal{T}_{\hat{d}}} \exp\bigl(\beta\, s_{u,d,k'}(q)\bigr)},
\end{equation}
where $\beta$ is the softmax temperature. 
The exploration coefficient $\alpha_{\mathrm{ucb}}$ controls the balance between exploiting the current point estimate and probing under-sampled arms. Compared with the frequency prior, the bandit prior is sensitive to query context (through $\boldsymbol{\phi}$), and, more importantly, maintains \emph{calibrated uncertainty}: the second term of $s_{u,d,k}(q)$ grows with epistemic uncertainty and shrinks as data accumulate. Both properties become consequential when the user's underlying preference distribution $\pi_u(d)$ is stochastic. 

We emphasize that neither prior is itself a contribution: both are well-established estimators, and alternatives such as Thompson sampling or kernel bandits are equally compatible with the architecture. Our claim is that \emph{any} consistent local harness, when paired with the override channel described next, suffices for personalized skill selection; comparing two instantiations characterizes how the strength of the statistical primitive interacts with the rest of the system. 



\subsection{Skill Selection via Statistical Priors}
\label{ssec:tool-selection-via-statistical-priors}
We now describe the end-to-end decision procedure that the \textsc{Local Harness} architecture induces, given a chosen statistical prior $h\in\{\mathrm{freq},\mathrm{bandit}\}$. The procedure separates each round of personalized skill selection into the following three steps. 

\paragraph{Step~1: Shared domain classification.}
Upon receiving a query $q_t$, the system performs a single LLM call to infer the target domain $\hat d(q_t)\in\mathcal{D}$. 
The resulting label $\hat d(q_t)$ restricts the candidate set to $\mathcal{T}_{\hat d(q_t)}\subset\mathcal{T}$ for all downstream operations.

\paragraph{Step~2: Local statistical default.}
The local harness consumes the inferred domain together with the user's state and returns a candidate action: 
\begin{equation}
\tilde a_t \;=\; \arg\max_{k\in\mathcal{T}_{\hat d(q_t)}} \hat p^{\,h}_{u,\hat d(q_t)}(k), 
\nonumber
\label{eq:harness_default}
\end{equation}
where $\hat p^{\,h}$ is either the success-rate distribution of Equation~\eqref{eq:freq_prior} or the bandit-derived distribution implied by Equation~\eqref{eq:ucb_score}. This step is fully local, deterministic, and conditional on the user state. 

\paragraph{Step~3: Semantic override probe.}
The remote LLM is engaged only to check whether the user's query contains an explicit lexical instruction that supersedes the user's habitual preference. We issue a single binary probe (details in \cref{sec:query-templates}): 
\begin{equation}
o_t,\,\tau_t \;=\; \mathrm{LLM}_{\text{override}}\!\left(q_t,\, \mathcal{T}_{\hat d(q_t)}\right), 
\nonumber
\label{eq:override_probe}
\end{equation}
where $o_t\in\{0,1\}$ indicates whether the query explicitly names a specific skill and $\tau_t\in\mathcal{T}_{\hat d(q_t)}$ is the named skill when $o_t=1$. 
\paragraph{Illustrative example.} Consider the cappuccino scenario in~\cref{fig:introduction}. The user, who habitually prefers \textsc{VibeCof'ing}, issues the query ``\textit{Order me a cappuccino}''. In Step~1, the shared domain classifier maps the query to the \texttt{eCommerce} domain, restricting the candidate set to the six coffee-ordering skills. In Step~2, the local harness consults its accumulated statistics for this user and returns $\tilde{a}_t = \textsc{VibeCof'ing}$ as the default action (since the user has historically rewarded this skill most often in this domain). In Step~3, the override probe inspects the query text: it contains no explicit skill name, so $o_t = 0$. The final action $a_t = \tilde{a}_t = \textsc{VibeCof'ing}$ is executed. Had the user instead said ``\textit{Order me a cappuccino \underline{from HouseBrew}}'', Step~3 would have returned $o_t = 1,\, \tau_t = \textsc{HouseBrew}$, and the override would have superseded the habitual default. 

\paragraph{Action and update.}
The action executed at round $t$ is then
\begin{equation}
a_t \;=\;
\begin{cases}
\tau_t & \text{if } o_t = 1,\\[2pt]
\tilde a_t & \text{otherwise.}
\end{cases}
\nonumber
\label{eq:final_action}
\end{equation}
The environment returns the binary reward $r_t=\mathds{1}\{a_t=a^*_t\}$, after which the local harness alone is updated with the tuple $(q_t,\hat d(q_t),a_t,r_t)$. 

\section{Theoretical Analysis}
\label{sec:theoretical-analysis}

We now show that, under mild assumptions automatically satisfied by the priors of~\cref{sec:method}, the \textsc{Local Harness} architecture provably brings the agent's expected per-round regret strictly closer to the user-optimal level. 

\paragraph{Setup.}
Fix a user $u$. At round $t$, the shared domain classifier reduces the candidate set to $\mathcal{T}_{\hat{d}}$ of size $K$. Write $Q(\cdot\mid q) := \pi_u(\hat{d}(q)) \in \Delta(\mathcal{T}_{\hat{d}})$ for the user's latent preference distribution; the ground-truth optimal skill on a standard query satisfies $a^{\star} \sim Q(\cdot\mid q)$. Any agent induces a (stochastic) policy with conditional distribution $\tilde{P}(\cdot\mid q) \in \Delta(\mathcal{T}_{\hat{d}})$, and the binary loss $\ell(a,a^{\star}) = \mathds{1}\{a \neq a^{\star}\}$ gives the expected per-round regret: 
\begin{equation}
\label{eq:risk-def}
R(\tilde{P}) \;:=\; 1 - \mathbb{E}_q
\big\langle \tilde{P}(\cdot\mid q),\, Q(\cdot\mid q)\big\rangle.
\end{equation}
We distinguish three distributions in $\Delta(\mathcal{T}_{\hat{d}})$: $P$ denotes the LLM's selection distribution; 
$h_t$ denotes the distribution induced by the local statistical prior $\hat{p}_{u,\hat{d}}^{h}$ of Section~\ref{sec:method} after $t$ rounds; and $P'_t$ denotes the policy of \textsc{Local Harness}. The override mechanism makes $P'_t$ a convex combination: 
\begin{equation}
\label{eq:mixture}
P'_t \;=\; (1-\lambda)\, h_t \;+\; \lambda\, P,
\qquad \lambda \in [0,1],
\end{equation}
with $\lambda = \rho_e$ for \textsc{Bandit-as-Override} and \textsc{Freq-as-Override}, since standard queries (fraction $1-\rho_e$) are routed through $h_t$ and explicit queries (fraction $\rho_e$) are routed through $P$. Note that $R(\tilde{P})$ is linear in $\tilde{P}$. 
Throughout, $\mathrm{TV}(P_1,P_2) = \tfrac{1}{2}\lVert P_1 - P_2\rVert_1$ denotes the total-variation distance between distributions $P_1$ and $P_2$. 
We then present the main result as follows. 

\begin{theorem}[Risk improvement under \textsc{Local Harness}]
\label{thm:main}
Assume (i) the local statistical prior is consistent, i.e.
$\mathbb{E}[\mathrm{TV}(h_t,Q)] \to 0$ as $t \to \infty$, and
(ii) the personalization problem is non-trivial, i.e.
$|R(P) - R(Q)| > 0$. Then there exists $T_0 < \infty$ such that for
every $t \ge T_0$,
\begin{equation}
\label{eq:main}
\mathbb{E}\big|R(P'_t) - R(Q)\big| \;<\; \big|R(P) - R(Q)\big|.
\end{equation}
\end{theorem}

\paragraph{Discussion.}
\Cref{thm:main} 
precisely characterizes the override architecture but is \emph{not} available to memory-augmented agents that funnel both statistical and semantic signals through a single LLM call. 
A quantitative refinement,
\begin{equation}
\label{eq:quant}
\begin{aligned}
\mathbb{E}\big|R(P'_t) - R(Q)\big|\;&\le\; \lambda\,|R(P) - R(Q)| \\ &+ 2(1-\lambda)\,C_{\text{prior}}\,t^{-1/2},
\end{aligned}
\end{equation}
exposes a two-phase convergence: an early phase dominated by the $O(t^{-1/2})$ exploration term, and a long-horizon plateau at $\lambda|R(P) - R(Q)|$, i.e.\ a strict $(1-\lambda)\times$ reduction over the LLM-only baseline. The full proof is deferred to Appendix~\ref{appendix:proofs}. 

\paragraph{Relaxing consistency.} Assumption~(i) is a law-of-large-numbers-level requirement that holds automatically in the stationary setting with clean binary feedback, but it can fail when feedback is noisy or the user's preference drifts. It is also stronger than the argument needs: the local prior never has to become perfect, only to end up \emph{less wrong} than the LLM. We therefore relax (i) to an asymptotic bias tolerance $\delta \ge 0$, i.e.\ $\limsup_{t\to\infty}\mathbb{E}[\mathrm{TV}(h_t,Q)] \le \delta$, and show in~\cref{cor:delta} that the strict improvement survives verbatim whenever $2\delta < |R(P)-R(Q)|$, with~\cref{thm:main} recovered at $\delta=0$. \Cref{appendix:proofs} reads off $\delta$ and the inflated constant $C_{\text{prior}}$ for the noisy, sparse, and drifting regimes, and~\cref{sec:robustness} verifies the resulting ordering empirically. 
\section{Experiment}
\label{sec:experiment}

We conduct extensive empirical evaluations to validate the effectiveness of the proposed \textsc{Local Harness} architecture across diverse language model backbones and preference regimes. 

\subsection{Experimental Setup}
\label{ssec:experimental-setup}
This subsection outlines the experimental setup, detailing the sandbox construction, evaluation metrics, model configurations, and the nine evaluated agents grouped by four different design families. We kindly refer the readers to the detailed implementations in~\cref{app:experimental-setup-details}.

\subsubsection{Sandbox}
Since preference-driven skill selection under implicit user preference is a newly identified problem, no existing benchmark provides synthetic users with controllable preference distributions over a realistic skill inventory. 
To evaluate personalized skill selection, we construct a simulation environment, \textsc{ToolBench-60}, comprising 60 skills across 10 domains curated from \textsc{ToolBench}~\citep{qin2023toolllm}. We model synthetic users by assigning each a latent preference distribution $\pi_{u}$ (ranging from deterministic one-hot to stochastic Dirichlet) to govern their habitual choices. These users are evaluated against a mixed query pool: \textit{standard queries} that omit specific skill names to mandate preference recovery, and \textit{explicit queries} that directly name a skill to test zero-shot semantic override capabilities. Full configuration details are provided in \cref{ssec:benchmark}. 

\subsubsection{Evaluation Protocol and Backbones}
\paragraph{Evaluation Metrics.} 
We systematically evaluate the agents' performance and preference alignment using four principal metrics:
(1) \textbf{Cumulative Regret (Regret):} The standard online learning metric, representing the cumulative cost of suboptimal skill selections over the $T$ interaction rounds.
(2) \textbf{Test Accuracy (Acc.):} The accuracy evaluated on the held-out test pool at the end of the interaction horizon, measuring the agent's ability to generalize learned preferences to unseen standard and explicit queries without further state updates.
(3) \textbf{Recovery Rate (R.R.):} Exclusively reported in the \emph{One-hot} preference regime, this metric computes the top-1 hit rate, indicating whether the argmax of the agent's learned probability distribution $\hat{p}_{u,d}$ correctly matches the ground-truth preferred skill $a^*_{u,d}$ across all user-domain pairs in all domains $\mathcal{D}$:
\begin{equation}
    \text{R.R.} = \frac{1}{N_u \cdot |\mathcal{D}|}\sum_{u,d} \mathds{1}\Big[\arg\max_k \hat{p}_{u,d}(k) = a^*_{u,d}\Big], 
    \nonumber
\end{equation}
where $a^*_{u,d} := \arg\max_{k \in \mathcal{T}_d} \pi_u(d)(k)$ denotes user $u$'s preferred skill in domain $d$. 
(4) \textbf{Spearman Rank Correlation (SRC):} Exclusively reported in the \emph{Soft} preference regime, this metric assesses the holistic alignment between the predicted and actual preference distributions by measuring their rank correlation:
\begin{equation}
    \text{SRC} = \frac{1}{N_u \cdot |\mathcal{D}|}\sum_{u,d} \rho\big(\text{rank}(\hat{p}_{u,d}),\ \text{rank}(p^*_{u,d})\big), 
    \nonumber
\end{equation}
where $\rho$ is the Spearman correlation coefficient. 

\paragraph{Backbones.} 
To ensure architectural robustness across varying model capabilities, we evaluate all agents using 3 diverse LLM backbones: GPT-5.2~\citep{gpt52}, DeepSeek-V4-Flash~\citep{deepseekai2026deepseekv4}, and Qwen3-30B-Instruct~\citep{yang2025qwen3technicalreport}. The experimental setup involves $N_u=50$ users interacting over $T=500$ rounds, with results averaged across $S=3$ independent seeds. 
For details about the hyperparameter configurations and API settings, please refer to \cref{ssec:hparams} and \cref{ssec:hardware}, respectively. 

\subsubsection{Agents Under Evaluation}


We instantiate \emph{nine} agents grouped into four design families to isolate the contribution of each architectural component. Agents within a family share a common decision primitive and differ only in incidental implementation choices; agents across families are intended to exhibit qualitatively different trade-offs. 
Detailed implementation for these agents is provided in \cref{ssec:agents}. 

\paragraph{Family I: No learning.}
\textbf{(1) \textsc{Random}} selects a skill uniformly at random from the full skill inventory. 
\textbf{(2) \textsc{ZeroShot-LLM}} prompts an LLM to select a skill using only the user's query and the skill descriptions, without any historical user state. 

\paragraph{Family II: Statistical only.}
These agents do not invoke the LLM; they decide purely from statistical priors. 
\textbf{(1) \textsc{Freq-Greedy}} greedily selects the skill with the \emph{frequency prior}. 
\textbf{(2) \textsc{Pure-Bandit}} uses the \emph{bandit prior} to calculate and directly select the skill with the highest UCB score.

\paragraph{Family III: LLM with Memory.}
These agents use the LLM as the sole decision-maker, differing only in the amount of per-user state exposed to the prompt. 
\textbf{(1) \textsc{InContext-Memory}} prompts an LLM to select a skill by appending the user's five most recent successful skill selections to the standard prompt. 
\textbf{(2) \textsc{Profile-Memory}} prompts an LLM to select a skill by injecting a complete, serialized log of the user's past successes and attempts for every skill directly into the context. 

\begin{table*}[tp]
\centering
\resizebox{\linewidth}{!}{
\begin{tabular}{cccccccc}
\hline
\multicolumn{2}{c}{}                                                                                                                      & \multicolumn{3}{c}{\textbf{one-hot}}                                                                                                               & \multicolumn{3}{c}{\textbf{soft-0.3}}                                                                                                                             \\ \cline{3-8} 
\multicolumn{2}{c}{\multirow{-2}{*}{\textit{\textbf{Qwen3-30B-Instruct}}}}                                                                & Regret($\downarrow$)                      & Acc.(\%, $\uparrow$)                                    & R.R.(\%, $\uparrow$)                         & Regret($\downarrow$)                                     & Acc.(\%, $\uparrow$)                                    & SRC($\uparrow$)                              \\ \hline
                                                                                                 & Random                                 & \cellcolor[HTML]{EFEFEF}492.1 ($\pm$ 0.2) & \cellcolor[HTML]{EFEFEF}1.7 ($\pm$ 0.1)                 & \cellcolor[HTML]{EFEFEF}15.9                 & \cellcolor[HTML]{EFEFEF}491.2 ($\pm$ 0.8)                & \cellcolor[HTML]{EFEFEF}1.7 ($\pm$ 0.3)                 & \cellcolor[HTML]{EFEFEF}0.000                \\
\multirow{-2}{*}{\textbf{No Learning}}                                                           & ZeroShot-LLM                           & 377.2 ($\pm$ 3.0)                         & 23.9 ($\pm$ 0.2)                                        & 15.9                                         & 377.5 ($\pm$ 3.6)                                        & 24.3 ($\pm$ 0.6)                                        & 0.000                                        \\ \hline
                                                                                                 & Freq-Greedy                            & \cellcolor[HTML]{EFEFEF}167.6 ($\pm$ 6.0) & \cellcolor[HTML]{EFEFEF}72.1 ($\pm$ 1.1)                & \cellcolor[HTML]{EFEFEF}86.4                 & \cellcolor[HTML]{EFEFEF}328.5 ($\pm$ 1.2)                & \cellcolor[HTML]{EFEFEF}33.3 ($\pm$ 1.0)                & \cellcolor[HTML]{EFEFEF}0.399                \\
\multirow{-2}{*}{\textbf{Statistical}}                                                           & Pure-Bandit                            & 140.2 ($\pm$ 1.9)                         & 80.4 ($\pm$ 0.6)                                        & 99.9                                         & 282.0 ($\pm$ 2.2)                                        & 39.5 ($\pm$ 1.7)                                        & {\ul \textbf{0.539}}                         \\ \hline
                                                                                                 & InContext-Memory                        & \cellcolor[HTML]{EFEFEF}363.9 ($\pm$ 3.8) & \cellcolor[HTML]{EFEFEF}27.2 ($\pm$ 0.2)                & \cellcolor[HTML]{EFEFEF}62.5                 & \cellcolor[HTML]{EFEFEF}372.7 ($\pm$ 3.7)                & \cellcolor[HTML]{EFEFEF}25.4 ($\pm$ 0.7)                & \cellcolor[HTML]{EFEFEF}0.282                \\
\multirow{-2}{*}{\textbf{\begin{tabular}[c]{@{}c@{}}LLM with \\ Memory\end{tabular}}}            & Profile-Memory                         & 269.5 ($\pm$ 4.5)                         & 53.4 ($\pm$ 0.6)                                        & 70.9                                         & 344.2 ($\pm$ 4.3)                                        & 32.9 ($\pm$ 0.7)                                        & 0.271                                        \\ \hline
                                                                                                 & Bandit-as-Context                      & \cellcolor[HTML]{EFEFEF}344.2 ($\pm$ 1.5) & \cellcolor[HTML]{EFEFEF}34.1 ($\pm$ 0.3)                & \cellcolor[HTML]{EFEFEF}82.7                 & \cellcolor[HTML]{EFEFEF}369.6 ($\pm$ 2.9)                & \cellcolor[HTML]{EFEFEF}26.4 ($\pm$ 0.6)                & \cellcolor[HTML]{EFEFEF}0.373                \\
                                                                                                 & Freq-as-Override                       & {\ul \textbf{126.3 ($\pm$ 4.5)}}          & 82.5 ($\pm$ 1.2)                                        & 92.5                                         & 295.3 ($\pm$ 2.9)                                        & 41.6 ($\pm$ 0.6)                                        & 0.288                                        \\
\multirow{-3}{*}{\textbf{\begin{tabular}[c]{@{}c@{}}LLM with \\ Statistical Prior\end{tabular}}} & \multicolumn{1}{l}{Bandit-as-Override} & \cellcolor[HTML]{EFEFEF}135.7 ($\pm$ 0.7) & \cellcolor[HTML]{EFEFEF}{\ul \textbf{84.3 ($\pm$ 0.7)}} & \cellcolor[HTML]{EFEFEF}{\ul \textbf{100.0}} & \cellcolor[HTML]{EFEFEF}{\ul \textbf{264.8 ($\pm$ 2.4)}} & \cellcolor[HTML]{EFEFEF}{\ul \textbf{46.2 ($\pm$ 0.6)}} & \cellcolor[HTML]{EFEFEF}{\ul \textbf{0.539}} \\ \hline
\end{tabular}
}
\caption{\textbf{Aggregate performance of nine skill-selection agents evaluated across varying user preference regimes on Qwen3-30B-Instruct.} Performance is measured by Cumulative Regret (Regret, $\downarrow$), Test Accuracy on the held-out pool (Acc., $\uparrow$), Recovery Rate (R.R., $\uparrow$), and Spearman Rank Correlation (SRC, $\uparrow$). 
}

\label{tab:main-results}
\end{table*}

\paragraph{Family IV: LLM with Statistical Prior.}
These agents combine a learned statistical estimator with an LLM to balance historical priors with semantic reasoning. 
\textbf{(1) \textsc{Bandit-as-Context}} prompts an LLM to make the final selection by providing the user's query alongside the probability distribution outputted by the \emph{bandit prior}. 
\textbf{(2) \textsc{Freq-as-Override}} operates as the procedure described in~\cref{ssec:tool-selection-via-statistical-priors} with \emph{frequency prior}. 
\textbf{(3) \textsc{Bandit-as-Override}} operates as the procedure described in~\cref{ssec:tool-selection-via-statistical-priors} with \emph{bandit prior}. 

\subsection{Main Results}
\label{ssec:main-results}


\Cref{tab:main-results} presents the aggregate performance on Qwen3-30B-Instruct; results on DeepSeek-V4-Flash and GPT-5.2 appear in~\cref{sec:more-exp-main-results}. We synthesize three primary findings: 

\paragraph{Neither Semantics Nor Statistics Alone Suffices.} The \textsc{No Learning} family consistently exhibits the highest cumulative regret, confirming that zero-shot prompting cannot deduce latent habits without historical signals. Conversely, purely statistical agents (\textsc{Freq-Greedy}, \textsc{Pure-Bandit}) attain competitive Regret and R.R. but are bottlenecked on Accuracy by their inability to process explicit overrides. Both dimensions are indispensable.

\paragraph{Exploration Outperforms Frequency Counting.} Bandit-based primitives systematically outperform frequency counting in R.R. and SRC, with the gap widening under the \textit{Soft} regime that better reflects stochastic real-world behaviors. Principled exploration is necessary for unearthing nuanced preferences within limited interaction rounds.

\paragraph{Decoupling Beats Prompt-Injected Memory.} The \textsc{LLM With Statistical Prior} family consistently achieves the lowest Regret and highest Test Accuracy across all backbones, while memory-augmented baselines (\textsc{InContext-Memory}, \textsc{Profile-Memory}), the prevailing paradigm in commercial agents, incur significantly higher regret. The difference is not in \textit{what} is remembered but in \textit{what makes the decision}. 

\begin{figure}[t]
    \centering
    \includegraphics[width=\linewidth]{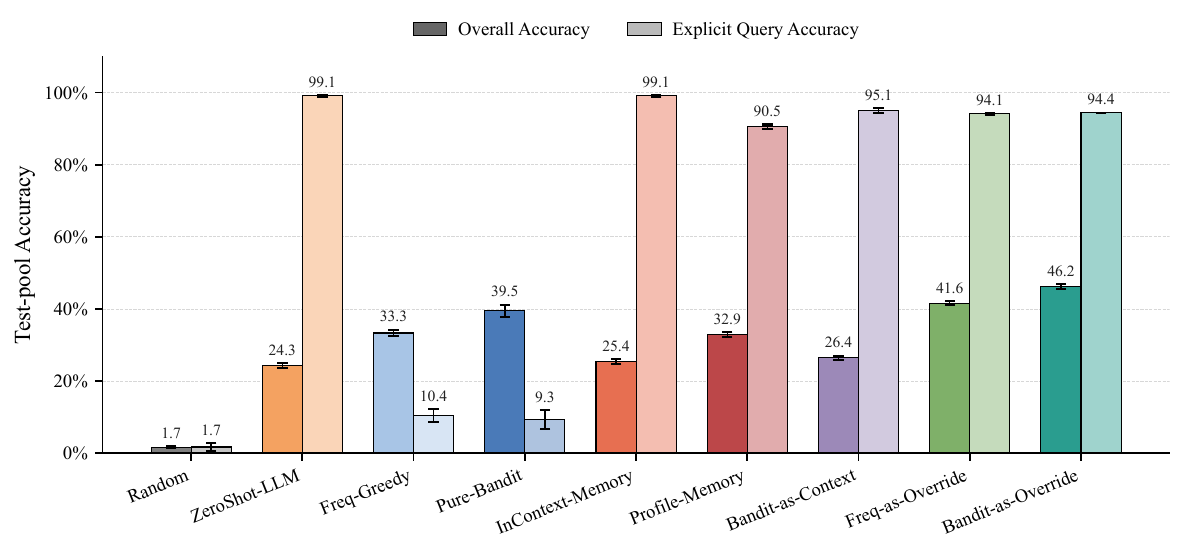}
    \caption{
    \textbf{Performance breakdown on explicit queries using the Qwen3-30B-Instruct backbone (Soft-0.3 regime).} Unlike purely statistical agents that fail on zero-shot instructions, our \textsc{Local Harness} achieves near-perfect execution by reserving the LLM exclusively as a semantic exception handler. 
    }
    \label{fig:test-pool-acc}
    \vspace{-10pt}
\end{figure}

\begin{figure*}[t]
    \centering
    \includegraphics[width=\linewidth]{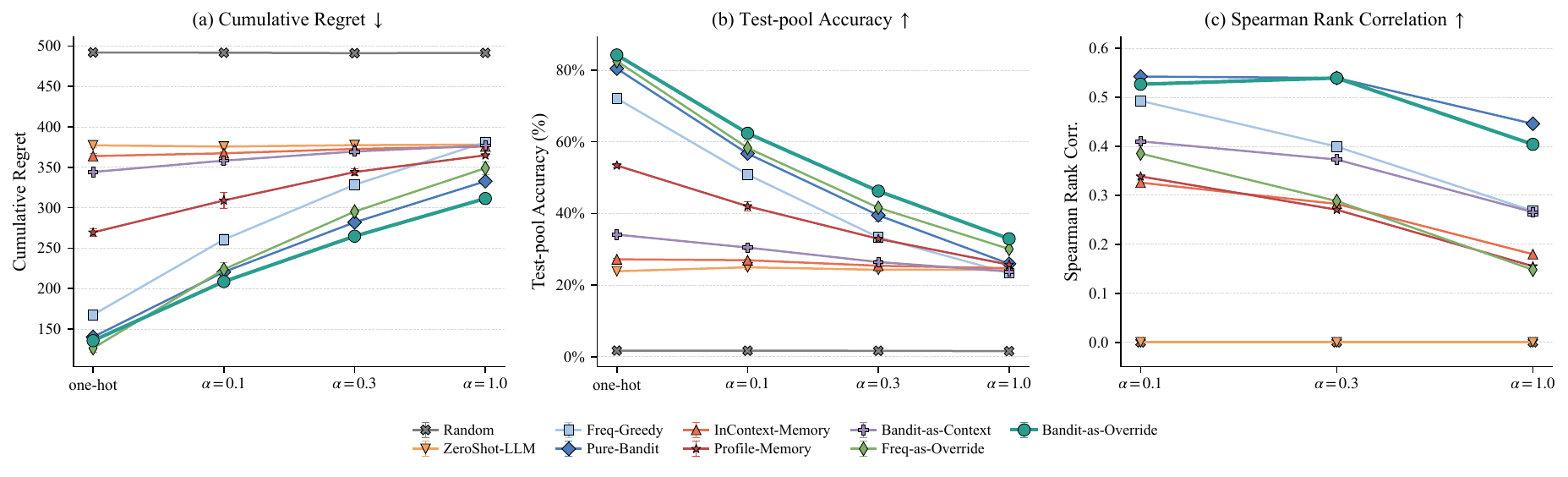}
    \caption{\textbf{Performance of evaluated agents across varying levels of user preference evenness ($\alpha$) using the Qwen3-30B-Instruct backbone.} The Dirichlet concentration parameter $\alpha$ sweeps from a deterministic \textit{one-hot} preference to a uniform distribution at $\alpha=1.0$. The subfigures report (a) Cumulative Regret ($\downarrow$), (b) Test-pool Accuracy ($\uparrow$), and (c) Spearman Rank Correlation ($\uparrow$). Across all evenness levels, \textsc{Bandit-as-Override} uniquely and consistently attains the lowest regret and highest test-pool accuracy. 
    }
    \label{fig:evenness-of-user-preference}
    \vspace{-10pt}
\end{figure*}

\subsection{Fine-Grained Evaluation}
\label{ssec:fine-grained-evaluation}

Beyond the following results, we also present analysis of cumulative regret in~\cref{sec:more-exp-regret}, rolling accuracy in~\cref{sec:more-exp-rolling-acc}, per-domain accuracy in~\cref{sec:more-exp-per-domain-acc}, and preference recovery in~\cref{sec:more-exp-soft-pref-recovery}. 

\subsubsection{Detailed Accuracy Analysis}

To provide a detailed accuracy analysis, we present a performance breakdown on explicit queries evaluated on Qwen3-30B-Instruct in \cref{fig:test-pool-acc}. 
We also present the results on DeepSeek-V4-Flash and GPT-5.2 via \cref{fig:test-pool-acc-dpsk} and \cref{fig:test-pool-acc-gpt52} in \cref{sec:more-exp-test-pool-acc}, respectively. 
These results highlight a critical vulnerability in purely statistical agents (\textsc{Freq-Greedy}, \textsc{Pure-Bandit}): while adept at modeling latent habits, their accuracy drops precipitously on explicit queries because they cannot process zero-shot semantic instructions. Furthermore, memory-augmented LLMs struggle to balance dense statistical tracking with these overrides. Our \textsc{Local Harness} architecture elegantly resolves this by delegating historical priors to local and reserving the LLM strictly as a semantic exception handler, achieving near-perfect execution on explicit instructions and validating our decoupled design.

\subsubsection{Evenness of User Preference}
\label{ssec:evenness}

To evaluate the impact of user preference evenness, we sweep the Dirichlet concentration parameter $\alpha \in \{\text{one-hot}, 0.1, 0.3, 1.0\}$. Our analysis yields three primary findings. First, \textsc{Bandit-as-Override} consistently maintains the lowest cumulative regret and highest test accuracy across all preference distributions. Second, the performance advantage of \textsc{Bandit-as-Override} over the exploration-free \textsc{Freq-as-Override} widens monotonically as preferences transition from deterministic to highly stochastic (increasing $\alpha$). This underscores that principled exploration is critical for mitigating the underestimation bias inherent in greedy frequency counting under stochastic regimes. Finally, the learned posterior of \textsc{Bandit-as-Override} closely tracks that of \textsc{Pure-Bandit}, verifying that our decoupled LLM semantic override channel effectively preserves the integrity of the statistical learning signal.

\subsection{Robustness and Scope}
\label{ssec:robustness-and-scope}
We further stress the sandbox along the axes that separate it from deployment conditions; protocols and tables are in the appendix. (i)~Under reward flip-noise ($p{=}0.2$), sparse feedback ($50\%$ of rounds), and an abrupt preference drift at round $250$ (\cref{sec:robustness}). (ii)~Re-rendering the same prior as percentages, raw counts, verbal grades, or an instruction-hardened variant (\cref{sec:why-incontext-fails}). (iii)~With hardened memory baselines (\cref{sec:hardened-memory}). We kindly refer the readers to the corresponding appendix for more details. 

\subsection{Correspondence with the Theory}
\label{ssec:theory-correspondence}
\Cref{thm:main} makes three qualitative predictions that align with our empirical findings. \textbf{(i) Mixture decomposition is necessary, not incidental.} The proof relies on the explicit convex-combination form $P'_t = (1-\lambda)h_t + \lambda P$ both for the convexity argument of~\cref{lem:mixing} and for the linearity argument that combines the two lemmas. Agents that route the statistical signal and the semantic signal through a single LLM call (\textsc{Bandit-as-Context}, \textsc{Profile-Memory}, \textsc{InContext-Memory}) do not admit such a decomposition, and the bound does not cover them, which is consistent with their markedly higher regret across all three backbones. 
\textbf{(ii) Backbone-invariance of the ranking.} Because~\cref{thm:main} depends on the LLM only through the constant $|R(P) - R(Q)|$ and not through any model-specific quantity, it predicts that the relative ordering of methods should be preserved across backbones, while absolute regret levels may shift. This is precisely what we observe across varying backbones in the main results. 
\textbf{(iii) Sensitivity to preference evenness.}
The frequency prior is purely exploitative and therefore more sensitive than bandit to the evenness of $Q$: as the preference distribution flattens, distinguishing the optimal skill from its near-ties requires increasingly fine-grained estimates, inflating $C_{\text{prior}}$. 
The bound therefore predicts a widening gap between \textsc{Bandit-as-Override} and \textsc{Freq-as-Override} as $\alpha$ increases. \Cref{fig:evenness-of-user-preference} confirms this: $\Delta_{\text{regret}}$ grows monotonically from $9.4$ at one-hot to $38$ at $\alpha\!=\!1.0$. 
The theory thus accounts for both the architectural ranking and its directional sensitivity to preference structure. 
\section{Conclusion}
\label{sec:conclusion}
We introduced \textsc{Local Harness}, an architecture for preference-driven skill selection in locally deployed personal agents that strictly decouples statistical preference learning from semantic intent parsing. 
By delegating user habit modeling to a lightweight local estimator and reserving the remote LLM as a semantic exception handler, our framework synergizes statistical personalization with semantic reasoning, establishing a robust paradigm for personalized agent designs.

\section*{Limitations}
While the \textsc{Local Harness} architecture successfully demonstrates the advantages of decoupling statistical preference learning from semantic intent parsing, our study has several limitations. 
First, since preference-driven skill selection is a newly identified problem with no existing simulation environment, we constructed \textsc{ToolBench-60} as a first dedicated sandbox to enable principled evaluation. While this benchmark already captures the core dynamics of implicit preferences across diverse domains and preference regimes, it currently models stationary user profiles; extending it to capture non-stationary temporal shifts and richer contextual dependencies, as well as broadening to additional benchmarks, is a natural next step that our open-sourced framework readily supports. 
Second, the current sequential decision-making formulation assumes immediate, explicit binary reward signals, though practical user feedback is frequently sparse, delayed, or noisy. Third, to ensure the local harness remains lightweight and free of external model dependencies, we employ deterministic feature hashing, which may lack the representational capacity of dense neural text embeddings to capture highly nuanced syntactic variations. Finally, while we removed the remote LLM from the high-frequency execution critical path, the framework still relies on capable remote foundation models to process the semantic override probe; evaluating its viability when paired with smaller on-device models remains an important area for future exploration. 
\section*{Acknowledgements}
This research was supported by National Natural Science Foundation of China (No.62476277), National Key Research and Development Program of China (No. 2024YFE0203200), CCF-ALIMAMA TECH Kangaroo Fund (No.CCF-ALIMAMA OF 2024008), and Huawei-Renmin University joint program on Information Retrieval. We also acknowledge the support provided by the fund for building worldclass universities (disciplines) of Renmin University of China and by the funds from Beijing Key Laboratory of Big Data Management and Analysis Methods, Gaoling School of Artificial Intelligence, Renmin University of China, from Engineering Research Center of Next-Generation Intelligent Search and Recommendation, Ministry of Education, from Intelligent Social Governance Interdisciplinary Platform, Major Innovation \& Planning Interdisciplinary Platform for the “DoubleFirst Class” Initiative, Renmin University of China, from Public Policy and Decision-making Research Lab of Renmin University of China, and from Public Computing Cloud, Renmin University of China. 
This research was also supported by the Outstanding Innovative Talents Cultivation Funded Programs 2025 of Renmin University of China. 

\bibliography{reference}

\appendix

\section{Experimental Setup Details}
\label{app:experimental-setup-details}
\subsection{Sandbox Construction}
\label{ssec:benchmark}

\paragraph{Skill Inventory.}
To ground the evaluation in realistic ecosystems, we derive our skill inventory from \textsc{ToolBench}~\cite{qin2023toolllm}, treating a tool as a simulation of a skill. 
From this corpus we curate a balanced subset of $60$ skills spanning $10$ domains (\textit{Finance}, \textit{Sports}, \textit{Travel}, \textit{Entertainment}, \textit{Gaming}, \textit{Education}, \textit{Communication}, \textit{Location}, \textit{eCommerce}, \textit{Social}), with $K=6$ skills per domain. We name this sandbox \textsc{ToolBench-60} in the later contents. Each skill is paired with a short natural-language description, 
exposed to language-model agents as part of the candidate description. 

\paragraph{Query Templates.} 
For each domain, we use an LLM ($\mathrm{GPT}$-$5.2$~\citep{gpt52}) to author two complementary template banks:
\textbf{(1) Standard templates}: 
natural-language user queries that describe an information need \emph{without} naming any specific skill. The correct skill in a domain for such a query is determined entirely by the user's latent preference. Standard templates are shared across users in the same domain; the correct answer differs across users. 
\textbf{(2) Explicit templates}: 
queries that \emph{explicitly} name a specific skill by quoting it inside the request (e.g., ``Use \textit{Twelve Data} to pull the last five years of daily prices for EUR/USD\ldots'').  An explicit query establishes a ground-truth target that supersedes the user's habitual choice. These queries enable us to measure whether an agent can break out of a learned habit when the user demands a non-default behavior. 
A detailed query generation process is presented in \cref{sec:query-templates}. 

\paragraph{Synthetic Users}
A user persona is defined by a preference function $\pi_u: \mathcal{D} \to \Delta(\mathcal{T}_d)$ that maps each domain to a distribution over its $K$ skills. We instantiate two preference regimes:
\textbf{(1) One-hot preference.} $\pi_u(d)$ places unit mass on a single preferred skill, sampled uniformly. This models a user with strong, deterministic habits and is the simpler benchmark setting. 
\textbf{(2) Soft preference} (Dirichlet). $\pi_u(d) \sim \mathrm{Dir}(\alpha\mathbf{1}_K)$ with concentration $\alpha$. Smaller $\alpha$ yields more peaked distributions (approaching one-hot), larger $\alpha$ yields flatter distributions. This models a user whose preference is graded rather than absolute. 
We use $\alpha = 0.3$ throughout the main experiments. 
We maintain a training pool of 200 queries and a held-out test pool of 50 per user, with a fixed explicit query ratio of $\rho_{e}=0.10$. 


\paragraph{Per-user query pool.}
For each user we construct two disjoint pools: a \emph{training pool} of size $|\mathcal{P}^{\text{tr}}_u| = 200$ and a held-out \emph{test pool} of size $|\mathcal{P}^{\text{te}}_u| = 50$, both sampled from the same template banks with an explicit query ratio of $\rho_{\text{e}} = 0.10$. Concretely, each pool contains $90\%$ standard queries and $10\%$ explicit queries:
\begin{itemize}
    \item For a \emph{standard} query, we uniformly draw a domain $d$ and a template from the standard bank for $d$. The ground-truth skill is sampled from $\pi_u(d)$ (one-hot mode collapses this to the preferred skill; soft mode samples per query). 
    \item For an \emph{explicit} query, we uniformly draw a domain $d$, draw an override target $\tau$ from $\mathcal{T}_d$ that is not the user's mode skill $\arg\max \pi_u(d)$, and sample a template from the explicit bank associated with $\tau$. The ground-truth skill is~$\tau$. 
\end{itemize}
The two pools are sampled independently with the same procedure but different random draws. The test pool is used only at the end of training (no updates are applied to any agent), eliminating the train--test leakage that would otherwise contaminate online metrics.

\paragraph{Reward signal.}
The reward at round $t$ is binary, $r_t = \mathds{1}\{a_t = a^\star_t\}$, where $a^\star_t$ is the ground-truth skill. 

\paragraph{Shared domain classifier.}
A na\"ive instantiation would expose the ground-truth domain label to agents that consume it (\textsc{Pure-Bandit}, \textsc{Freq-Greedy}, the hybrid agents), giving them an unrealistic advantage. To avoid this, we prepend a \emph{shared} domain classifier to every round: a single LLM call, using the same backbone as the agents, labels the query domain $\hat{d}(q_t)$, and \emph{every} agent observes this label. Agents that do not require an explicit domain (e.g.\ \textsc{Random}, \textsc{ZeroShot-LLM}) are unaffected. This design (i) ensures fair comparison across statistical and language-model agents, (ii) amortizes one LLM call across all agents in a round, and (iii) lets us report domain-classification accuracy as a separate diagnostic ($\approx 94.5\%$ on the \textsc{ToolBench}-60 sandbox with \textsc{GPT}-$5.2$). 

\paragraph{Per-domain classifier confusion.} The per-domain accuracy heatmaps of~\cref{sec:more-exp-per-domain-acc} show that \emph{all} agents perform comparatively poorly on the Education and eCommerce domains. Mining the shared classifier's predictions from our main-experiment logs (GPT-5.2, soft-preference regime, $3$ seeds) explains this quantitatively:

\begin{center}
\resizebox{\linewidth}{!}{
\small
\begin{tabular}{lcl}
\hline
\textbf{Domain} & \textbf{Acc.} & \textbf{Dominant confusion} \\ \hline
eCommerce  & 72.4\% & $\to$ Finance, $\to$ Communication \\
Education  & 82.4\% & $\to$ Entertainment \\
Finance    & 96.0\% & $\to$ eCommerce \\
Other $7$  & 96.7--100\% & --- \\ \hline
Overall    & 94.5\% & \\ \hline
\end{tabular}
}
\end{center}

\noindent The root cause is cross-domain skill overlap inherited from \textsc{ToolBench}. The eCommerce domain of \textsc{ToolBench-60} contains cryptocurrency-price and gold/silver-price APIs, while the Finance domain also contains crypto and DeFi skills, so cryptocurrency- and precious-metal-related queries are legitimately ambiguous between the two domains. Similarly, the Education domain's reading skills (e.g.\ a digital-library API) generate queries superficially indistinguishable from Entertainment.

This does not affect the fairness of any between-agent comparison. The domain classifier is a shared upstream component consumed identically by every agent, so its inaccuracy on these two domains lowers all methods uniformly. Indeed, in the per-domain heatmaps the override family remains the best-performing family within both problematic columns, i.e.\ the relative ranking is preserved. 

\subsection{Hyperparameter Configuration}
\label{ssec:hparams}

\paragraph{Online evaluation protocol.}
Unless otherwise noted, every experiment uses $N_u = 50$ users per seed, $T = 500$ interaction rounds per user, and $S = 3$ independent random seeds. Round budget is chosen so that each $(\text{user}, \text{domain})$ pair receives, in expectation, $T / |\mathcal{D}| = 50$ interactions. 

\paragraph{Bandit hyperparameters.}
All bandit-based agents (\textsc{Pure-Bandit}, \textsc{Bandit-as-Context}, \textsc{Bandit-as-Override}) use \textsc{LinUCB}~\cite{li2010contextual} with exploration coefficient $\alpha_{\text{ucb}} = 1.0$ and one independent arm per $(\text{user}, \text{domain}, \text{skill})$ triple. The context vector is a $96$-dimensional feature-hashed representation~\cite{weinberger2009feature} constructed by concatenating three sign-hashed sub-vectors: $64$ dimensions for the query string, $16$ for the skill name, and $16$ for the (inferred) domain. Feature hashing avoids vocabulary construction and is robust to the diverse, multi-domain natural-language queries in \textsc{ToolBench-60}. We adopt hashing rather than a learned text encoder (e.g.\ a sentence transformer) by design: the local harness is intended to run on commodity hardware without external model dependencies, and hashing is the canonical lightweight encoder for online linear estimators in this regime.  

\paragraph{Bandit-to-LLM interface.}
When the bandit's posterior must be exposed to a language model (\textsc{Bandit-as-Context}), we convert the per-skill UCB scores $\{s_1,\ldots,s_K\}$ to a probability vector via tempered softmax, $p_i = \exp(\beta s_i) / \sum_j \exp(\beta s_j)$, with temperature $\beta = 3.0$. We chose this value empirically. 

\paragraph{LLM decoding.}
All LLM calls (domain classification, agent selection, and the override yes/no probe used by \textsc{Bandit-as-Override} and \textsc{Freq-as-Override}) use temperature $\tau_{\text{llm}} = 0.0$ for deterministic decoding. 
Each prompt requests a strict JSON object as output. A response that fails JSON parsing falls back to regular-expression and substring matching against the candidate skill list; on total failure, the first candidate is returned. Empirically fall-back fires on $\lesssim 0.5\%$ of calls with the evaluated backbones. 

\paragraph{Memory-based baselines.} 
The recency-weighted baseline \textsc{InContext-Memory} retains the most recent $K_m = 5$ successful selections per $(\text{user}, \text{domain})$ pair. The structured baseline \textsc{Profile-Memory} maintains exact $(\text{successes}, \text{attempts})$ counts per $(\text{user}, \text{domain}, \text{skill})$ and renders them into the LLM prompt as a deterministic profile string. 

\paragraph{Override probe.} 
The override-based agents (\textsc{Bandit-as-Override}, \textsc{Freq-as-Override}) issue a single binary probe to the LLM per round, structured to elicit a strict JSON answer of the form \verb|{"override": true, "tool": "<name>"}| or \verb|{"override": false}|. The probe lists the domain's skill names but omits their descriptions. 

\subsection{LLM Call and API}
\label{ssec:hardware}

\paragraph{Language-model backbones.}
A central methodological commitment of this work is that any architectural claim must be validated across LLM backbones of differing strength. We therefore evaluate all agents on the following three models, spanning capability tiers:
\begin{itemize}
    \item \textsc{GPT-5.2}~\citep{gpt52} (OpenAI; frontier-tier general-purpose model)
    \item \textsc{DeepSeek-V4-Flash}~\citep{deepseekai2026deepseekv4} (DeepSeek; mid-tier cost-optimized instruction-tuned model)
    \item \textsc{Qwen3-30B-A3B-Instruct-2507}~\citep{yang2025qwen3technicalreport} (Alibaba; mid-tier instruction-tuned MoE model)

\end{itemize}
This selection covers both general-purpose and reasoning-specialized models at different capability tiers, demonstrating that the proposed architecture's advantage holds across all three, which is essential to our methodological claim. The same prompts and parsing logic are used across different LLM backbones. 

\paragraph{Intra-round parallelism.}
Within a single round, up to seven LLM calls are required: one for the shared domain classifier and one each for the up-to-six language-model agents. We dispatch these in parallel via a per-simulation thread pool, preserving deterministic semantics by (i) keeping all non-LLM agents on the main thread to fix the global random-number-generator interleaving, (ii) submitting each LLM agent's $\texttt{select\_tool}$ call as an independent future, and (iii) collecting results and applying agent state updates in the original agent order after every future has resolved. 

\paragraph{Reproducibility.}
Each experiment is fully described by the tuple (benchmark file, backbone name, seed list, preference mode, and the remaining hyperparameters listed in Section~\ref{ssec:hparams}). All sources of stochasticity (user generation, query sampling, LinUCB context hashing, and the LLM decoder) are either explicitly seeded or deterministic. A two-level checkpointing scheme (per-seed and per-user-within-seed) lets an interrupted run resume without recomputation; resumed runs are byte-identical to their uninterrupted counterparts. Code, benchmark JSON files, and a reference run script are released with the paper. 

\paragraph{Compute budget.}
A single main run ($N_u=50$, $T=500$, $S=3$) issues approximately $3.75\times 10^5$ LLM calls per backbone per preference regime (one for domain classification plus six for the LLM-using agents, across $N_u\cdot T \cdot S$ rounds, plus $\sim 1.5\times 10^4$ calls for the held-out test evaluation). At typical hosted-API latencies our parallelized implementation completes one such run in $6$--$10$~hours of wall-clock time per backbone. The full set of main experiments and ablations reported in this paper consumes approximately $200$~hours of cumulative API time. 

\subsection{Agents Under Evaluation}
\label{ssec:agents}

We instantiate \emph{nine} agents grouped into four design families to isolate the contribution of each architectural component. Agents within a family share a common decision primitive and differ only in incidental implementation choices; agents across families are intended to exhibit qualitatively different trade-offs. 

\paragraph{Family I: No learning.}
\begin{itemize}
    \item \textbf{\textsc{Random}} selects a skill uniformly at random from the full inventory $\mathcal{T}$. It serves as a basic baseline. 
    \item \textbf{\textsc{ZeroShot-LLM}} prompts the LLM with the full skill inventory (names plus their textual descriptions) and the user's query, with no per-user state. It is the natural lower bound for any LLM-based personalized agent. 
\end{itemize}

\paragraph{Family II: Statistical only.}
These agents do not invoke the LLM at decision time; they decide
purely from the (action, reward) stream accumulated for the current
user.
\begin{itemize}
    \item \textbf{\textsc{Freq-Greedy}} maintains a per-$(\text{user}, \text{domain}, \text{skill})$ success-rate table and greedily selects the skill with the highest empirical rate; ties and untried skills are broken by a single round of round-robin initialization. It tests whether na\"ive frequency counting suffices in the absence of exploration. 
    \item \textbf{\textsc{Pure-Bandit}} runs \textsc{LinUCB} as described in Section~\ref{ssec:hparams}. At each round, it computes a UCB score for each skill in the (inferred) domain and selects the argmax. The contrast with \textsc{Freq-Greedy} isolates the value of principled exploration; the contrast with the hybrid agents below isolates the value of LLM-driven override.
\end{itemize}

\paragraph{Family III: LLM with Memory.}
These agents use the LLM as the primary (and sole) decision-maker; they differ only in how much per-user state they expose to the prompt. This family is included specifically to represent the architectural choices behind current production AI agents.
\begin{itemize}

    \item \textbf{\textsc{InContext-Memory}} additionally appends the $K_m = 5$ most recent successful selections for the current $(\text{user}, \text{domain})$ to the prompt, modeling a recency-weighted personalization signal of the kind that emerges in long chat sessions.

    \item \textbf{\textsc{Profile-Memory}} maintains an exact bookkeeping of $(\text{successes}, \text{attempts})$ per $(\text{user}, \text{domain}, \text{skill})$ across the entire horizon and serializes it into the prompt as a structured profile, e.g.\ ``\textit{Tool X: 8/10 successful (80\%)}.'' This is, to our knowledge, the most common LLM-with-memory baseline currently feasible: it discards no information and is supplied to the model in a parsed, lossless form. It is also the most direct analogue of production memory-augmented agents such as OpenClaw~\citep{openclaw} and Claude Code~\citep{claude_code}, in which structured user facts are written to a persistent store and re-injected into the LLM context at decision time. 
\end{itemize}

\paragraph{Family IV: LLM with Statistical Prior.}
These agents combine a learned statistical estimator with an LLM in different ways; they are the locus of our methodological contribution.
\begin{itemize}
    \item \textbf{\textsc{Bandit-as-Context}} exposes the bandit's posterior to the LLM as part of its selection prompt. Concretely, for each skill in the inferred domain, the prompt lists the skill name together with its tempered-softmax probability under the current \textsc{LinUCB} posterior (Section~\ref{ssec:hparams}). The LLM is asked to integrate this prior with the query semantics and emit a final selection. \textsc{Bandit-as-Context} is the canonical ``LLM-as-Bayesian-reasoner'' instantiation of the hybrid idea and represents the natural \emph{a priori} expectation for how a bandit and an LLM should be combined. 

    \item \textbf{\textsc{Freq-as-Override}} defaults to operate the same greedy frequency rule as \textsc{Freq-Greedy}. A separate, single LLM call asks a binary question: ``\textit{does the query explicitly name a tool by name?}''; if so, the LLM returns the named skill and overrides the frequency default. By holding the override mechanism constant and replacing \textsc{LinUCB} with frequency counting, this agent isolates the contribution of \emph{bandit exploration}. 

    \item \textbf{\textsc{Bandit-as-Override}} is identical in structure to \textsc{Freq-as-Override} except that the default selection is produced by \textsc{LinUCB}. The LLM is asked only the same binary override question, never to perform the primary selection. This realizes our central design hypothesis: the statistical estimator handles personalization (a credit-assignment problem) and the LLM is invoked only as a narrow exception handler for queries whose surface form supersedes the learned preference. 
\end{itemize}

\paragraph{Common interface.}
Every agent implements the same two-method interface: $\texttt{select\_tool}(q, \hat{d}, u, \mathcal{T}) \to a$ and $\texttt{update}(q, \hat{d}, u, a, r)$. State updates are performed only inside $\texttt{update}$; selection is read-only with respect to agent state. This shared contract makes every comparison in \cref{sec:experiment} an apples-to-apples swap of the decision primitive while holding the surrounding pipeline, domain classifier, query distribution, and evaluation protocol fixed.


\section{Proofs for~\cref{sec:theoretical-analysis}}
\label{appendix:proofs}

We use the notation introduced in~\cref{sec:theoretical-analysis}. All expectations $\mathbb{E}[\cdot]$ are over the randomness of the interaction history $\{(q_\tau, a_\tau, r_\tau)\}_{\tau < t}$ that generates $h_t$, unless otherwise specified. The proof of~\cref{thm:main} proceeds in two steps. \Cref{lem:tv} bounds the regret gap between any two policies by their total-variation distance; \Cref{lem:mixing} exploits the mixture structure of $P'_t$ to show that this distance is in expectation strictly below $\mathrm{TV}(P,Q)$. The strict separation in regret then follows from the linearity of $R$.

\subsection{Lemma~1: Regret--TV Transfer}

\begin{lemma}[Regret--TV transfer]
\label{lem:tv}
For any $\tilde{P}(\cdot\mid q) \in \Delta(\mathcal{T}_{\hat{d}})$,
\begin{equation}
\label{eq:lem1}
\big|R(\tilde{P}) - R(Q)\big|
\;\le\; 2\,\mathbb{E}_q\big[\mathrm{TV}(\tilde{P}(\cdot\mid q),\,
Q(\cdot\mid q))\big].
\end{equation}
\end{lemma}

\begin{proof}
By the linearity of $R$ in Equation~\eqref{eq:risk-def},
\[
R(\tilde{P}) - R(Q)
\;=\; \mathbb{E}_q\big\langle Q - \tilde{P},\, Q\big\rangle.
\]
Fix $q$ and apply H\"older's inequality:
\begin{equation}
\begin{aligned}
\big|\langle Q - \tilde{P},\, Q\rangle\big| \;&\le\; \|Q - \tilde{P}\|_1 \cdot \|Q\|_\infty \; \\ &\le\; \|Q - \tilde{P}\|_1 \; \\ &=\; 2\,\mathrm{TV}(\tilde{P}, Q),
\end{aligned}
\end{equation}
where the second inequality uses $\|Q\|_\infty \le 1$ (since $Q$ is a probability vector) and the final equality is the standard $L^1$ representation $\mathrm{TV}(P_1, P_2) = \tfrac{1}{2}\|P_1 - P_2\|_1$. Taking expectation over $q$ yields Equation~\eqref{eq:lem1}.
\end{proof}

This lemma is the discrete, $0/1$-loss specialization of the IPM-based (Integral Probability Metric) loss-transfer principle: with the function class $\mathcal{H} = \{h:\|h\|_\infty \le 1\}$ the integral probability metric reduces to TV, and the corresponding Lipschitz constant of the loss is at most $2$.

\subsection{Lemma~2: Mixing Reduces TV}

\begin{lemma}[Mixing reduces TV and convergence rate]
\label{lem:mixing}
Let $P'_t = (1-\lambda)\,h_t + \lambda\,P$ with $\lambda \in [0,1]$. Then: 

\noindent(a) \textbf{(Mixing reduces TV.)} For every $q$,
\begin{equation}
    \mathrm{TV}(P'_t, Q) \;\leq\; (1-\lambda)\,\mathrm{TV}(h_t, Q) + \lambda\,\mathrm{TV}(P, Q).
\end{equation}
Consequently, if $\mathbb{E}[\mathrm{TV}(h_t, Q)] \to 0$ and $\mathrm{TV}(P, Q) > 0$, there exists $T_0$ such that
\begin{equation}
    \mathbb{E}\!\left[\mathrm{TV}(P'_t, Q)\right] < \mathrm{TV}(P, Q) \qquad \forall\, t \geq T_0.
\end{equation}
(b) \textbf{(Parametric convergence rate.)} Both statistical priors of Section~\ref{sec:method} satisfy
\begin{equation}
    \mathbb{E}\!\left[\mathrm{TV}(h_t, Q)\right] \;\leq\; C_{\mathrm{prior}}\, t^{-1/2}
    \label{eq:tv-rate}
\end{equation}
for a problem-dependent constant $C_{\mathrm{prior}}$. In particular, this implies $\mathbb{E}[\mathrm{TV}(h_t, Q)] \to 0$ as $t \to \infty$.

\end{lemma}

\begin{proof}
\textbf{Part (a).} The map $\tilde{P} \mapsto \mathrm{TV}(\tilde{P}, Q) = \tfrac{1}{2}\|\tilde{P} - Q\|_1$ is convex, since the $L^1$ norm is convex and the affine map $\tilde{P} \mapsto \tilde{P} - Q$ preserves convexity. Applying Jensen's inequality to the convex combination $P'_t = (1-\lambda)h_t + \lambda P$ yields the first inequality. Taking expectation over the randomness of $h_t$ and subtracting $\mathrm{TV}(P, Q)$ from both sides,
\begin{equation}
\begin{aligned}
&\mathbb{E}[\mathrm{TV}(P'_t, Q)] - \mathrm{TV}(P, Q) \;\leq\; \\ &(1-\lambda)\bigl(\mathbb{E}[\mathrm{TV}(h_t, Q)] - \mathrm{TV}(P, Q)\bigr). 
\end{aligned}
\end{equation}

By the consistency assumption (which follows from Part (b)), $\mathbb{E}[\mathrm{TV}(h_t, Q)] < \mathrm{TV}(P, Q)$ for all sufficiently large $t$, making the right-hand side strictly negative.

\noindent\textbf{Part (b).} We verify the rate for both priors:
\begin{itemize}
    \item \emph{Frequency prior.} Each $(u, d, k')$ cell aggregates Bernoulli reward signals with mean $Q(k')$; Hoeffding's inequality combined with a union bound over the $K$ skills yields $\mathbb{E}[\mathrm{TV}(h^{\mathrm{freq}}_t, Q)] = O(\sqrt{K/t})$.
    \item \emph{Bandit prior.} The standard LinUCB regret bound \citep{li2010contextual} gives cumulative regret $O(d\sqrt{t \log t})$, which translates into a per-round sub-optimality gap of $\tilde{O}(\sqrt{d/t})$. A standard argument relating action sub-optimality to TV distance yields $\mathbb{E}[\mathrm{TV}(h^{\mathrm{bandit}}_t, Q)] = \tilde{O}(\sqrt{d/t})$.
\end{itemize}
Both bounds are summarized by Equation~\eqref{eq:tv-rate}. 
\end{proof}

\subsection{Proof of Theorem~\ref{thm:main}}

\begin{proof}
By the linearity of $R$ in Equation~\eqref{eq:risk-def} and the mixture form~\eqref{eq:mixture}, 
\begin{equation}
\begin{aligned}
R(P'_t) \;&=\; R\big((1-\lambda)h_t + \lambda P\big) \; \\ &=\; (1-\lambda)\, R(h_t) + \lambda\, R(P). 
\end{aligned}
\end{equation}
Subtracting $R(Q)$ from both sides and rearranging,
\resizebox{\linewidth}{!}{
$
R(P'_t) - R(Q)
\;=\; (1-\lambda)\big(R(h_t) - R(Q)\big)
\;+\; \lambda\big(R(P) - R(Q)\big).
$
}
Taking absolute value and applying the triangle inequality,
\begin{equation}
\label{eq:triangle}
\begin{aligned}
&\big|R(P'_t) - R(Q)\big| \;\le\; \\ &(1-\lambda)\,\big|R(h_t) - R(Q)\big| \;+\; \lambda\,\big|R(P) - R(Q)\big|. 
\end{aligned}
\end{equation}
Taking expectation over $h_t$ and applying Lemma~\ref{lem:tv}
to $\tilde{P} = h_t$,
\[
\mathbb{E}\,\big|R(h_t) - R(Q)\big|
\;\le\; 2\,\mathbb{E}\big[\mathrm{TV}(h_t, Q)\big].
\]
By~\cref{lem:mixing} (b), $\mathbb{E}[\mathrm{TV}(h_t,Q)] \to 0$,
so given the non-triviality assumption $|R(P) - R(Q)| > 0$ there
exists $T_0 < \infty$ such that
\begin{equation}
\label{eq:dagger}
2\,\mathbb{E}[\mathrm{TV}(h_t, Q)]
\;<\; |R(P) - R(Q)|
\qquad \forall\, t \ge T_0.
\end{equation}
Taking expectation over $h_t$ on both sides of Equation~\eqref{eq:triangle}, 
we obtain: 
\begin{align*}
&\mathbb{E}\,\big|R(P'_t) - R(Q)\big|\\
&\;\le\; 2(1-\lambda)\,\mathbb{E}[\mathrm{TV}(h_t, Q)]
\;+\; \lambda\,|R(P) - R(Q)| \\
&\;<\; (1-\lambda)\,|R(P) - R(Q)|
\;+\; \lambda\,|R(P) - R(Q)| \\
&\;=\; |R(P) - R(Q)|,
\end{align*}
where the second inequality uses Equation~\eqref{eq:dagger}. This is precisely Equation~\eqref{eq:main}. 
\end{proof}

\paragraph{Quantitative bound.}
Substituting the explicit rate $\mathbb{E}[\mathrm{TV}(h_t,Q)] \le C_{\text{prior}}\,t^{-1/2}$ into Equation~\eqref{eq:triangle} and using Lemma~\ref{lem:tv} on the first term yields Equation~\eqref{eq:quant}. The threshold $T_0$ scales as $T_0 = O\big(C_{\text{prior}}^2 / |R(P) - R(Q)|^2\big)$: the larger the LLM's miscalibration to the user, the fewer interactions are needed for \textsc{Local Harness} to overtake. 

\paragraph{Remark.}
The proof uses convexity of TV (\cref{lem:mixing}) and linearity of $R$ (Equation~\eqref{eq:risk-def}); both follow from the structural fact that $P'_t$ is a convex combination of $h_t$ and $P$. Architectures that route both the statistical prior and the semantic signal through a single LLM call (e.g.\ \textsc{Bandit-as-Context} or \textsc{Profile-Memory}) do \emph{not} produce a distribution of this form, because the LLM's non-linear processing destroys the convex combination. Theorem~\ref{thm:main} therefore identifies the \emph{mixture decomposition} as the structural property that distinguishes the \textsc{LLM-with-Statistical-Prior} family from its memory-augmented counterpart. 

\subsection{Relaxing Consistency: the $\delta$-Accurate Local Prior}
\label{ssec:delta-corollary}

\Cref{thm:main} assumes the local prior is consistent, which holds in the stationary sandbox with clean binary feedback but can fail under the noisy or non-stationary conditions we stress in~\cref{sec:robustness}. The assumption is stronger than the argument requires: the local prior never needs to become perfect, only to end up less wrong than the LLM. We therefore relax consistency to an asymptotic bias tolerance.

\begin{corollary}[$\delta$-accurate local prior]
\label{cor:delta}
Assume (i$'$) the local prior is asymptotically $\delta$-accurate, i.e.\ $\limsup_{t\to\infty}\mathbb{E}[\mathrm{TV}(h_t,Q)] \le \delta$ for some $\delta \ge 0$, and (ii) the personalization problem is non-trivial, i.e.\ $|R(P)-R(Q)| > 0$, as in~\cref{thm:main}. Then for every $t$,
\begin{equation}
\label{eq:delta-bound}
\begin{aligned}
\mathbb{E}\big|R(P'_t) - R(Q)\big| \;\le\;\; &\lambda\,|R(P)-R(Q)| \\ &+ 2(1-\lambda)\,\mathbb{E}[\mathrm{TV}(h_t,Q)],
\end{aligned}
\end{equation}
and if moreover $2\delta < |R(P)-R(Q)|$ and $\lambda < 1$, there exists $T_0 < \infty$ such that for every $t \ge T_0$,
\begin{equation}
\label{eq:delta-strict}
\mathbb{E}\big|R(P'_t) - R(Q)\big| \;<\; \big|R(P)-R(Q)\big|.
\end{equation}
\Cref{thm:main} is recovered as the special case $\delta = 0$.
\end{corollary}

\begin{proof}
By the linearity of $R$ noted after Equation~\eqref{eq:mixture},
\begin{equation}
\begin{aligned}
R(P'_t)-R(Q) \;=\;\; &(1-\lambda)\big(R(h_t)-R(Q)\big) \\ &+ \lambda\big(R(P)-R(Q)\big),
\end{aligned}
\nonumber
\end{equation}
so the triangle inequality bounds $\mathbb{E}|R(P'_t)-R(Q)|$ by $(1-\lambda)\,\mathbb{E}|R(h_t)-R(Q)| + \lambda\,|R(P)-R(Q)|$. Since $R(h_t)-R(Q) = \mathbb{E}_q\langle Q - h_t, Q\rangle$ by Equation~\eqref{eq:risk-def} and $\lVert Q(\cdot\mid q)\rVert_\infty \le 1$, H\"older's inequality gives $|R(h_t)-R(Q)| \le 2\,\mathrm{TV}(h_t,Q)$, exactly as in~\cref{lem:tv}; this yields Equation~\eqref{eq:delta-bound}. Under (i$'$), the $\limsup$ of the right-hand side of Equation~\eqref{eq:delta-bound} equals $\lambda|R(P)-R(Q)| + 2(1-\lambda)\delta$, which is strictly below $|R(P)-R(Q)|$ precisely when $2\delta < |R(P)-R(Q)|$ and $\lambda < 1$; the existence of $T_0$ follows. Relative to the proof of~\cref{thm:main}, only the mixture inequality of~\cref{lem:mixing}(a) changes, by the additive term $2(1-\lambda)\delta$.
\end{proof}

\paragraph{Reading off $\delta$ in three representative regimes.} The three corruptions of~\cref{sec:robustness} instantiate the corollary as follows.

\begin{itemize}
    \item \emph{Flip-noise $p < 1/2$: $\delta = O(p)$.} Every observed success rate becomes $p + (1-2p)\,Q(k)$, so flipping shrinks the gap between any two skills by a factor $(1-2p)$ while leaving the per-round sampling noise unchanged. Recovering the same ranking therefore requires estimates $1/(1-2p)$ times sharper, i.e.\ $C_{\text{prior}} \to C_{\text{prior}}/(1-2p)$ in Equation~\eqref{eq:quant}. Meanwhile the learned distribution converges to $Q$ blended with a $p$-weighted uniform component, whence $\delta = O(p)$; the guarantee survives whenever this bias is small relative to the LLM's miscalibration.
    \item \emph{Sparse feedback at rate $\rho$: $\delta = 0$.} Whether feedback arrives is independent of the choice made and of its outcome, so the observed rounds form an unbiased subsample and the estimator's limit is unchanged, i.e.\ consistency holds verbatim. The only cost is that $t$ rounds now carry $\rho t$ samples, and estimation error scales with the square root of the sample count: $1/\sqrt{\rho t} = (1/\sqrt{\rho})\,t^{-1/2}$, i.e.\ $C_{\text{prior}} \to C_{\text{prior}}/\sqrt{\rho}$.
    \item \emph{Preference drift: $\delta$ scales with the magnitude of the drift.} No static estimator can track a moving target. The standard remedy is a discounted or sliding-window local estimator, whose tracking error is bounded in terms of the amount of drift, the corollary then applies with that bound as $\delta$.
\end{itemize}

The two rate inflations are qualitatively different: $1/(1-2p)$ diverges as $p \to 1/2$, because flipped feedback actively corrupts information and carries none at $p=1/2$, whereas $1/\sqrt{\rho}$ grows only as a square root, because missing feedback merely withholds samples without biasing them. This predicts that noise should hurt more than sparsity at comparable corruption levels, which is the ordering observed in~\cref{tab:robustness-regret,tab:robustness-acc}. 

\section{More Main Results on Varying Backbones}
\label{sec:more-exp-main-results}

\begin{table*}[tp]
\centering
\resizebox{\linewidth}{!}{
\begin{tabular}{cccccccc}
\hline
\multicolumn{2}{c}{}                                                                                                                      & \multicolumn{3}{c}{\textbf{one-hot}}                                                                                                               & \multicolumn{3}{c}{\textbf{soft-0.3}}                                                                                                              \\ \cline{3-8} 
\multicolumn{2}{c}{\multirow{-2}{*}{\textit{\textbf{DeepSeek-V4-Flash}}}}                                                                 & Regret($\downarrow$)                      & Acc.(\%, $\uparrow$)                                    & R.R.(\%, $\uparrow$)                         & Regret($\downarrow$)                                     & Acc.(\%, $\uparrow$)                                    & SRC($\uparrow$)               \\ \hline
                                                                                                 & Random                                 & \cellcolor[HTML]{EFEFEF}492.0 ($\pm$ 0.3) & \cellcolor[HTML]{EFEFEF}1.7 ($\pm$ 0.2)                 & \cellcolor[HTML]{EFEFEF}15.9                 & \cellcolor[HTML]{EFEFEF}491.7 ($\pm$ 0.5)                & \cellcolor[HTML]{EFEFEF}1.6 ($\pm$ 0.3)                 & \cellcolor[HTML]{EFEFEF}0.000 \\
\multirow{-2}{*}{\textbf{No Learning}}                                                           & ZeroShot-LLM                           & 377.1 ($\pm$ 2.6)                         & 23.8 ($\pm$ 0.6)                                        & 15.9                                         & 376.7 ($\pm$ 4.1)                                        & 24.6 ($\pm$ 0.6)                                        & 0.000                         \\ \hline
                                                                                                 & Freq-Greedy                            & \cellcolor[HTML]{EFEFEF}166.0 ($\pm$ 3.0) & \cellcolor[HTML]{EFEFEF}73.0 ($\pm$ 0.5)                & \cellcolor[HTML]{EFEFEF}86.9                 & \cellcolor[HTML]{EFEFEF}328.0 ($\pm$ 2.0)                & \cellcolor[HTML]{EFEFEF}33.7 ($\pm$ 0.9)                & \cellcolor[HTML]{EFEFEF}0.402 \\
\multirow{-2}{*}{\textbf{Statistical}}                                                           & Pure-Bandit                            & 139.6 ($\pm$ 2.0)                         & 80.3 ($\pm$ 0.5)                                        & 99.8                                         & 280.7 ($\pm$ 2.2)                                        & 39.8 ($\pm$ 1.0)                                        & {\ul \textbf{0.549}}          \\ \hline
                                                                                                 & InContext-Memory                        & \cellcolor[HTML]{EFEFEF}365.4 ($\pm$ 2.6) & \cellcolor[HTML]{EFEFEF}26.8 ($\pm$ 0.0)                & \cellcolor[HTML]{EFEFEF}62.6                 & \cellcolor[HTML]{EFEFEF}372.5 ($\pm$ 3.8)                & \cellcolor[HTML]{EFEFEF}25.5 ($\pm$ 0.7)                & \cellcolor[HTML]{EFEFEF}0.294 \\
\multirow{-2}{*}{\textbf{\begin{tabular}[c]{@{}c@{}}LLM with \\ Memory\end{tabular}}}            & Profile-Memory                         & 245.1 ($\pm$ 5.2)                         & 57.1 ($\pm$ 1.5)                                        & 70.9                                         & 327.4 ($\pm$ 4.8)                                        & 34.8 ($\pm$ 0.7)                                        & 0.251                         \\ \hline
                                                                                                 & Bandit-as-Context                      & \cellcolor[HTML]{EFEFEF}216.6 ($\pm$ 4.6) & \cellcolor[HTML]{EFEFEF}62.9 ($\pm$ 1.4)                & \cellcolor[HTML]{EFEFEF}87.9                 & \cellcolor[HTML]{EFEFEF}309.6 ($\pm$ 5.2)                & \cellcolor[HTML]{EFEFEF}37.0 ($\pm$ 1.4)                & \cellcolor[HTML]{EFEFEF}0.465 \\
                                                                                                 & Freq-as-Override                       & {\ul \textbf{103.3 ($\pm$ 1.8)}}          & 87.0 ($\pm$ 1.0)                                        & 94.1                                         & 288.0 ($\pm$ 3.9)                                        & 41.6 ($\pm$ 0.9)                                        & 0.273                         \\
\multirow{-3}{*}{\textbf{\begin{tabular}[c]{@{}c@{}}LLM with \\ Statistical Prior\end{tabular}}} & \multicolumn{1}{l}{Bandit-as-Override} & \cellcolor[HTML]{EFEFEF}118.9 ($\pm$ 1.0) & \cellcolor[HTML]{EFEFEF}{\ul \textbf{87.4 ($\pm$ 0.4)}} & \cellcolor[HTML]{EFEFEF}{\ul \textbf{100.0}} & \cellcolor[HTML]{EFEFEF}{\ul \textbf{255.9 ($\pm$ 1.6)}} & \cellcolor[HTML]{EFEFEF}{\ul \textbf{47.2 ($\pm$ 1.1)}} & \cellcolor[HTML]{EFEFEF}0.523 \\ \hline
\end{tabular}
}
\caption{\textbf{Aggregate performance of nine skill-selection agents evaluated across varying user preference regimes on DeepSeek-V4-Flash.} Performance is measured by Cumulative Regret (Regret, $\downarrow$), Test Accuracy on the held-out pool (Acc., $\uparrow$), Recovery Rate (R.R., $\uparrow$), and Spearman Rank Correlation (SRC, $\uparrow$). 
}
\label{tab:main-exp-dpsk}
\end{table*}

\begin{table*}[tp]
\centering
\resizebox{\linewidth}{!}{
\begin{tabular}{cccccccc}
\hline
\multicolumn{2}{c}{}                                                                                                                      & \multicolumn{3}{c}{\textbf{one-hot}}                                                                                                               & \multicolumn{3}{c}{\textbf{soft-0.3}}                                                                                                              \\ \cline{3-8} 
\multicolumn{2}{c}{\multirow{-2}{*}{\textit{\textbf{GPT-5.2}}}}                                                                           & Regret($\downarrow$)                      & Acc.(\%, $\uparrow$)                                    & R.R.(\%, $\uparrow$)                         & Regret($\downarrow$)                                     & Acc.(\%, $\uparrow$)                                    & SRC($\uparrow$)               \\ \hline
                                                                                                 & Random                                 & \cellcolor[HTML]{EFEFEF}491.8 ($\pm$ 0.5) & \cellcolor[HTML]{EFEFEF}1.9 ($\pm$ 0.3)                 & \cellcolor[HTML]{EFEFEF}15.9                 & \cellcolor[HTML]{EFEFEF}491.5 ($\pm$ 0.5)                & \cellcolor[HTML]{EFEFEF}1.6 ($\pm$ 0.3)                 & \cellcolor[HTML]{EFEFEF}0.000 \\
\multirow{-2}{*}{\textbf{No Learning}}                                                           & ZeroShot-LLM                           & 376.0 ($\pm$ 2.6)                         & 24.3 ($\pm$ 0.6)                                        & 15.9                                         & 375.5 ($\pm$ 3.4)                                        & 24.8 ($\pm$ 1.0)                                        & 0.000                         \\ \hline
                                                                                                 & Freq-Greedy                            & \cellcolor[HTML]{EFEFEF}162.4 ($\pm$ 3.2) & \cellcolor[HTML]{EFEFEF}73.6 ($\pm$ 1.1)                & \cellcolor[HTML]{EFEFEF}87.3                 & \cellcolor[HTML]{EFEFEF}327.2 ($\pm$ 2.1)                & \cellcolor[HTML]{EFEFEF}34.0 ($\pm$ 0.7)                & \cellcolor[HTML]{EFEFEF}0.404 \\
\multirow{-2}{*}{\textbf{Statistical}}                                                           & Pure-Bandit                            & 137.1 ($\pm$ 1.0)                         & 81.0 ($\pm$ 0.6)                                        & 99.9                                         & 279.2 ($\pm$ 1.4)                                        & 40.3 ($\pm$ 2.0)                                        & {\ul \textbf{0.554}}          \\ \hline
                                                                                                 & InContext-Memory                        & \cellcolor[HTML]{EFEFEF}371.5 ($\pm$ 2.2) & \cellcolor[HTML]{EFEFEF}25.4 ($\pm$ 0.2)                & \cellcolor[HTML]{EFEFEF}63.5                 & \cellcolor[HTML]{EFEFEF}373.5 ($\pm$ 4.0)                & \cellcolor[HTML]{EFEFEF}25.2 ($\pm$ 0.7)                & \cellcolor[HTML]{EFEFEF}0.304 \\
\multirow{-2}{*}{\textbf{\begin{tabular}[c]{@{}c@{}}LLM with \\ Memory\end{tabular}}}            & Profile-Memory                         & 208.7 ($\pm$ 9.6)                         & 64.8 ($\pm$ 2.2)                                        & 67.5                                         & 316.0 ($\pm$ 3.6)                                        & 38.6 ($\pm$ 1.2)                                        & 0.193                         \\ \hline
                                                                                                 & Bandit-as-Context                      & \cellcolor[HTML]{EFEFEF}266.0 ($\pm$ 3.5) & \cellcolor[HTML]{EFEFEF}53.1 ($\pm$ 4.9)                & \cellcolor[HTML]{EFEFEF}94.1                 & \cellcolor[HTML]{EFEFEF}327.1 ($\pm$ 3.8)                & \cellcolor[HTML]{EFEFEF}35.0 ($\pm$ 0.5)                & \cellcolor[HTML]{EFEFEF}0.473 \\
                                                                                                 & Freq-as-Override                       & {\ul \textbf{107.2 ($\pm$ 3.1)}}          & 86.8 ($\pm$ 0.9)                                        & 94.2                                         & 286.5 ($\pm$ 1.9)                                        & 43.0 ($\pm$ 0.5)                                        & 0.286                         \\
\multirow{-3}{*}{\textbf{\begin{tabular}[c]{@{}c@{}}LLM with \\ Statistical Prior\end{tabular}}} & \multicolumn{1}{l}{Bandit-as-Override} & \cellcolor[HTML]{EFEFEF}122.4 ($\pm$ 1.6) & \cellcolor[HTML]{EFEFEF}{\ul \textbf{86.9 ($\pm$ 1.8)}} & \cellcolor[HTML]{EFEFEF}{\ul \textbf{100.0}} & \cellcolor[HTML]{EFEFEF}{\ul \textbf{256.8 ($\pm$ 1.4)}} & \cellcolor[HTML]{EFEFEF}{\ul \textbf{47.3 ($\pm$ 1.2)}} & \cellcolor[HTML]{EFEFEF}0.535 \\ \hline
\end{tabular}
}
\caption{\textbf{Aggregate performance of nine skill-selection agents evaluated across varying user preference regimes on GPT-5.2.} Performance is measured by Cumulative Regret (Regret, $\downarrow$), Test Accuracy on the held-out pool (Acc., $\uparrow$), Recovery Rate (R.R., $\uparrow$), and Spearman Rank Correlation (SRC, $\uparrow$). 
}
\label{tab:main-exp-gpt52}
\end{table*}

To ensure the architectural robustness of our framework across varying model capabilities, we present the aggregate performance evaluated on DeepSeek-V4-Flash and GPT-5.2 in \cref{tab:main-exp-dpsk} and \cref{tab:main-exp-gpt52}, respectively. Consistent with the primary findings, the results demonstrate that the \textsc{Bandit-as-Override} agent consistently achieves the lowest cumulative regret and the highest test accuracy across all backbones and preference regimes. Purely statistical agents establish strong preference alignment but remain bottlenecked in test accuracy due to their inability to process explicit semantic overrides. Similarly, memory-augmented LLMs yield significantly higher regret and lower accuracy compared to our decoupled approach. These empirical results reinforce the core conclusion: delegating probabilistic credit assignment to a local statistical prior while reserving the remote LLM strictly for complex intent parsing avoids the systemic failures of forcing a single model to manage both tasks simultaneously. 
\section{More Experiments for Test-Pool Accuracy}
\label{sec:more-exp-test-pool-acc}

\begin{figure}[t]
    \centering
    \includegraphics[width=\linewidth]{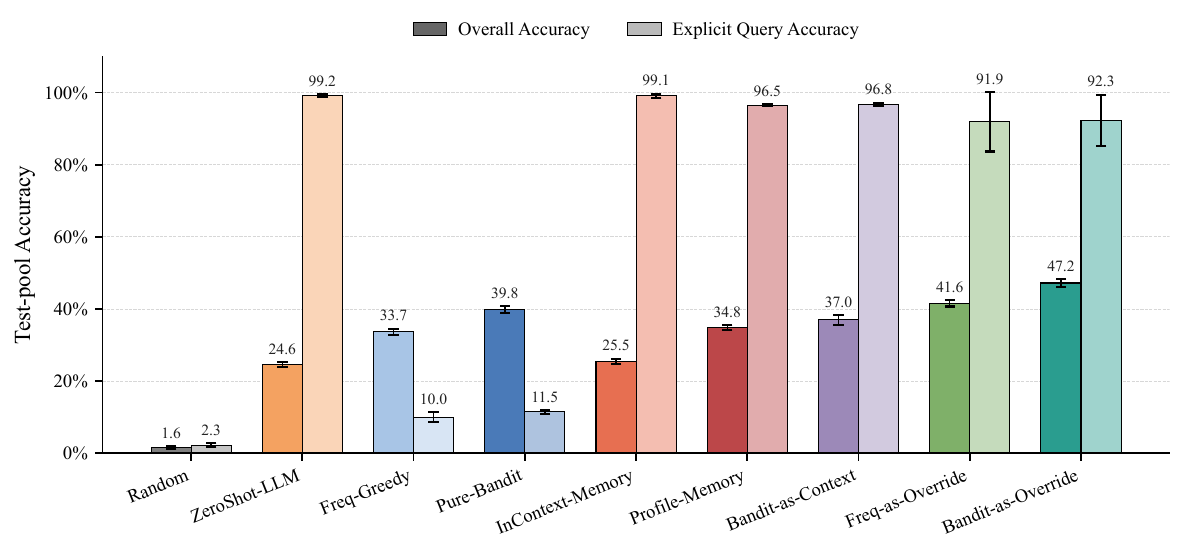}
    \caption{Performance breakdown on explicit queries evaluated using the DeepSeek-V4-Flash backbone under the Soft-0.3 preference regime. }
    \label{fig:test-pool-acc-dpsk}
    \vspace{-10pt}
\end{figure}

\begin{figure}[t]
    \centering
    \includegraphics[width=\linewidth]{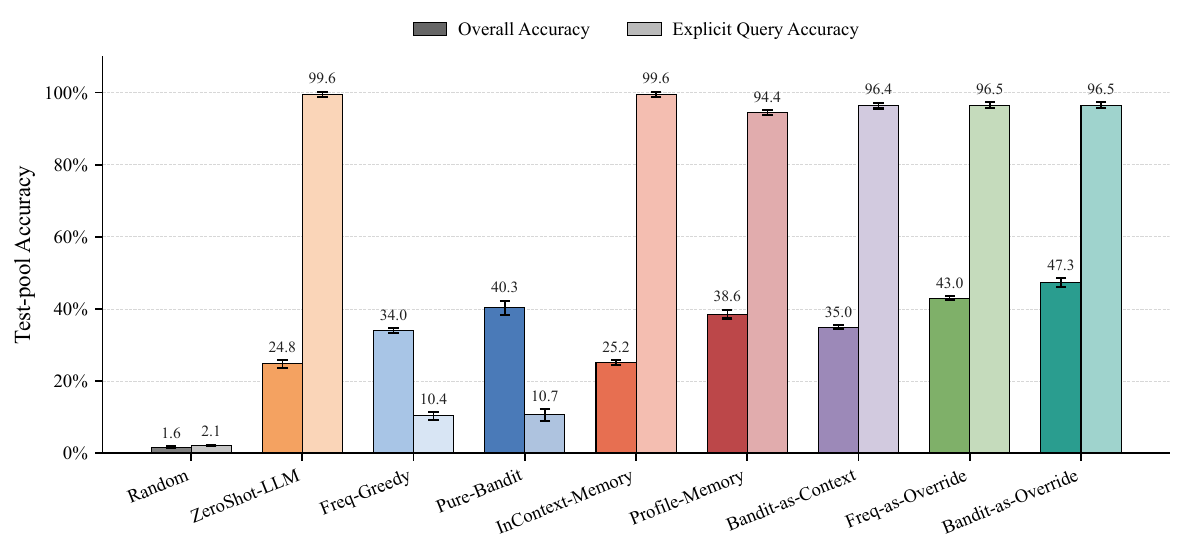}
    \caption{Performance breakdown on explicit queries evaluated using the GPT-5.2 backbone under the Soft-0.3 preference regime.}
    \label{fig:test-pool-acc-gpt52}
    \vspace{-10pt}
\end{figure}

\Cref{fig:test-pool-acc-dpsk} and \cref{fig:test-pool-acc-gpt52} provide a detailed accuracy breakdown on explicit queries utilizing the DeepSeek-V4-Flash and GPT-5.2 backbones under the Soft-0.3 preference regime. These fine-grained evaluations highlight a critical vulnerability in purely statistical agents, which suffer a precipitous drop in accuracy on explicit queries because they cannot process zero-shot semantic instructions. Furthermore, memory-augmented LLMs consistently struggle to balance dense statistical tracking with these semantic overrides. The \textsc{Local Harness} architecture elegantly resolves this limitation across diverse foundational models; by delegating historical priors to a local bandit and reserving the LLM strictly as a semantic exception handler, \textsc{Bandit-as-Override} achieves near-perfect execution on explicit instructions, thereby thoroughly validating our decoupled design. 
\section{Regret Analysis}
\label{sec:more-exp-regret}
\begin{figure*}[t]
    \centering
    \includegraphics[width=\linewidth]{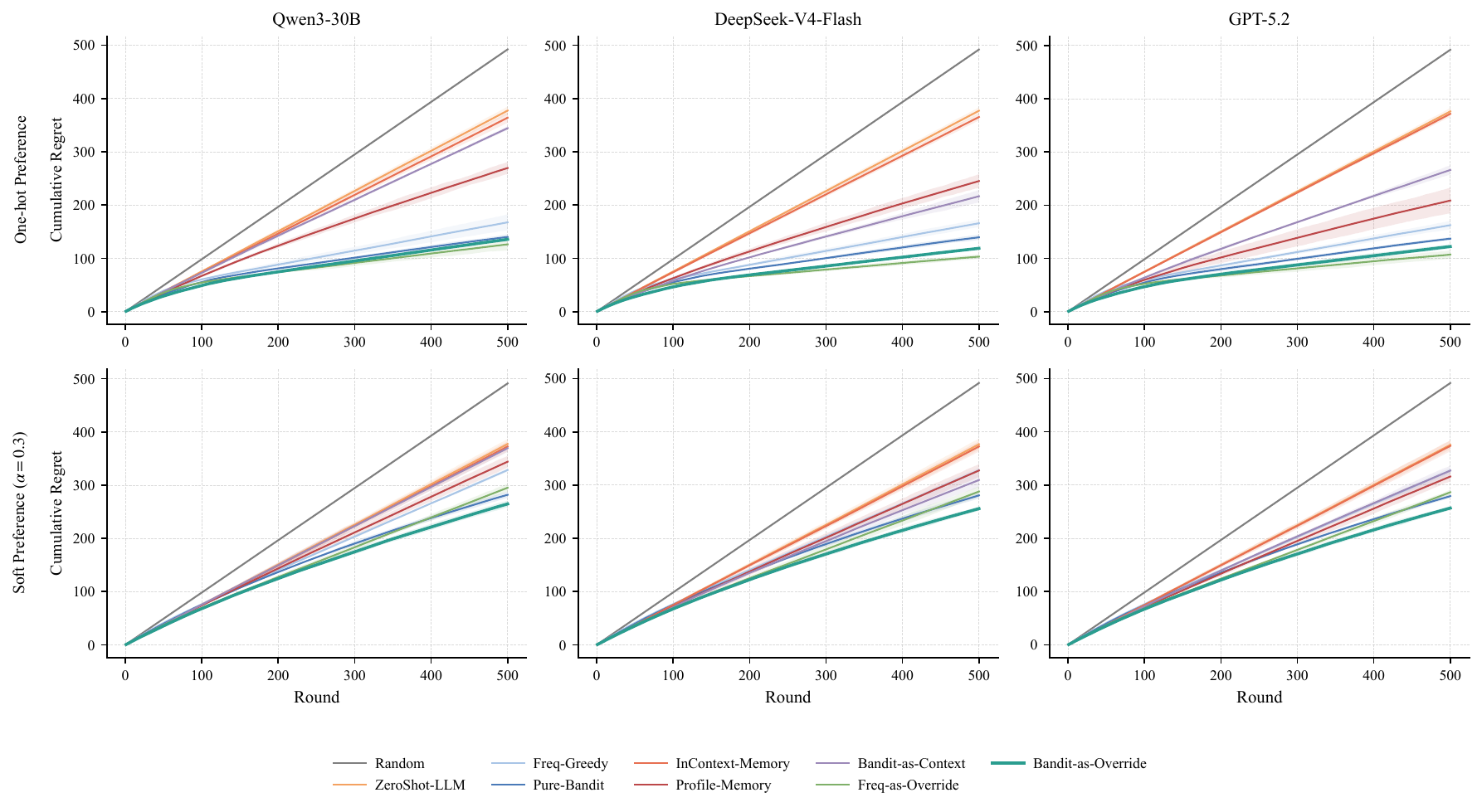}
    \caption{
    \textbf{Cumulative regret over the interaction horizon for the six main experiments.} Rows correspond to the preference regime (top: one-hot; bottom: soft preference with $\alpha=0.3$); columns correspond to the LLM backbone (\textsc{Qwen3-30B}, \textsc{DeepSeek-V4-Flash}, \textsc{GPT-5.2}). Each panel plots cumulative regret $\sum_{\tau\le t}(1-r_\tau)$, averaged across $N_u=50$ users and $S=3$ independent seeds; shaded bands denote $95\%$ confidence intervals. Lower curves indicate faster convergence and lower aggregate cost. Across all three backbones and both regimes, \textsc{Bandit-as-Override} (deep teal) attains the lowest regret throughout training, and its advantage over \textsc{Freq-as-Override} becomes more pronounced under the soft regime.}
    \label{fig:appendix_regret}
    \vspace{-10pt}
\end{figure*}

\Cref{fig:appendix_regret} reports the full cumulative-regret trajectories. Across all six panels, \textsc{Bandit-as-Override} and \textsc{Freq-as-Override} attain the lowest curve throughout training. The two memory-augmented LLM baselines (\textsc{Profile-Memory} and \textsc{InContext-Memory}) accumulate regret at a near-linear rate, indicating that prompt-injected per-user statistics do not substitute for a locally-fitted estimator. Consistent with the evenness analysis in \cref{ssec:evenness}, the gap between \textsc{Bandit-as-Override} and \textsc{Freq-as-Override} visibly widens between the top row (one-hot) and the bottom row ($\alpha=0.3$), confirming that the marginal value of exploration grows as user preference becomes stochastic.
\section{Rolling Accuracy}
\label{sec:more-exp-rolling-acc}

\begin{figure*}[t]
    \centering
    \includegraphics[width=\linewidth]{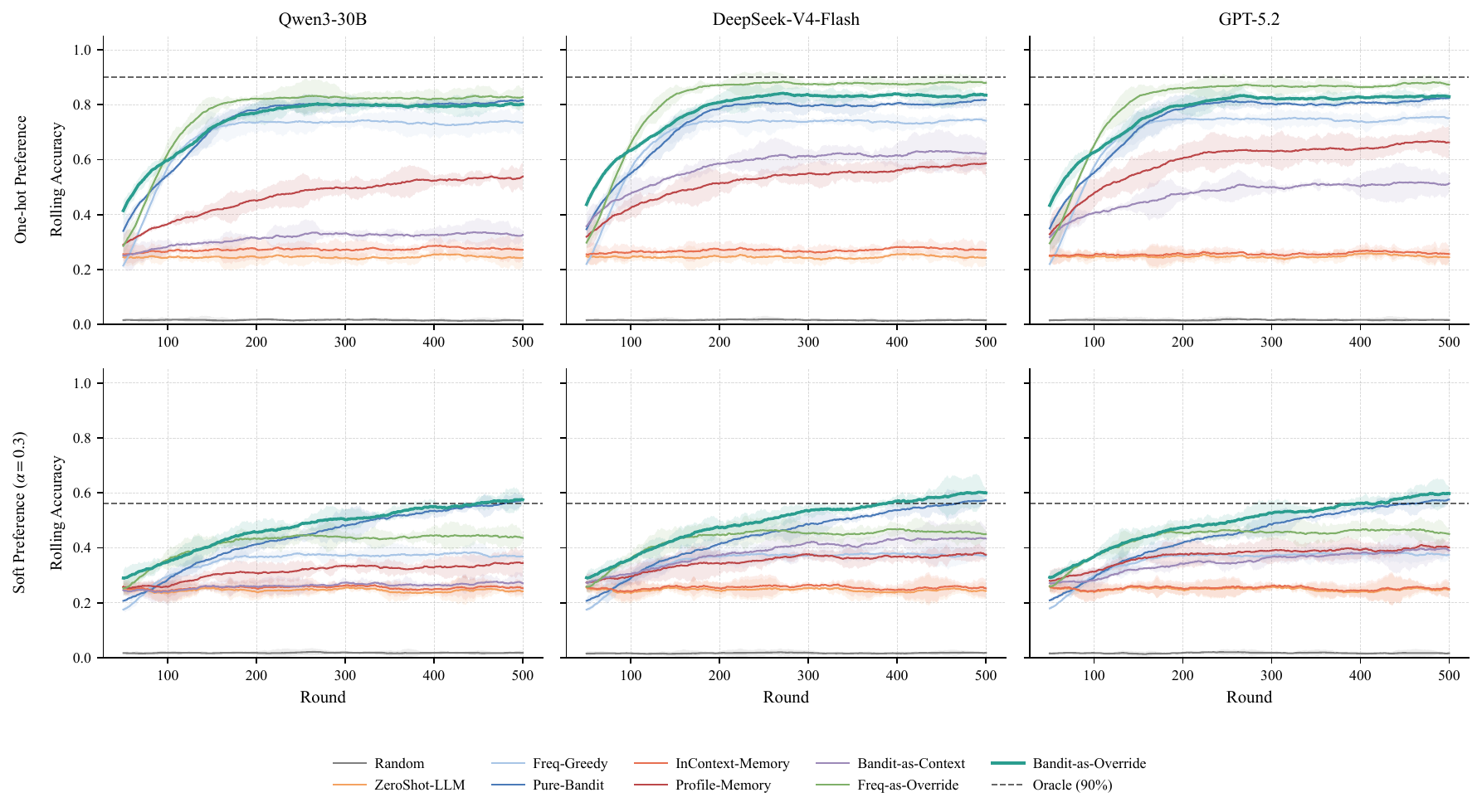}
    \caption{ \textbf{Rolling per-round accuracy along the interaction horizon for the six main experiments.} For each round $t$, we plot the mean reward over a sliding window of the previous $W=50$ rounds, averaged across $N_u=50$ users and $S=3$ seeds; shaded bands denote $95\%$ confidence intervals. 
    The dashed Oracle line marks the accuracy ceiling of \emph{any} static policy that commits to a single skill per $(u,d)$ pair; under one-hot preferences it equals $1-\rho_e=0.9$ by construction, while in the soft regime it is determined by the mode of each user's Dirichlet preference. 
    Within each panel, \textsc{Bandit-as-Override} converges fastest and plateaus closest to the Oracle line; the gap to memory-augmented LLM baselines (\textsc{Profile-Memory}, \textsc{InContext-Memory}) is large under one-hot preferences and remains substantial under the soft regime.}
    \label{fig:appendix_rolling_acc}
    \vspace{-10pt}
\end{figure*}

\Cref{fig:appendix_rolling_acc} traces the per-round convergence dynamics with a sliding window of $W=50$ rounds. The dashed Oracle line marks the accuracy ceiling of \emph{any} static policy that commits to a single skill per $(u,d)$ pair; under one-hot preferences it equals $1-\rho_{\text{e}}=0.9$ by construction, while in the soft regime it is determined by the mode of each user's Dirichlet preference. Two observations stand out. First, the statistical-only agents (\textsc{Freq-Greedy} and \textsc{Pure-Bandit}) plateau \emph{at} the Oracle line, exactly as predicted: lacking access to query semantics, they can at best recover the static optimum. Second, the override-based agents visibly cross above this ceiling, because their LLM probe can re-route explicit-query rounds to the named skill. 
The curves also show that \textsc{Bandit-as-Override} converges faster than every alternative across all panels, an advantage that is most pronounced in the early rounds where the preference signal is scarce.
\section{Per-Domain Accuracy Analysis}
\label{sec:more-exp-per-domain-acc}

\begin{figure*}[t]
    \centering
    \includegraphics[width=\linewidth]{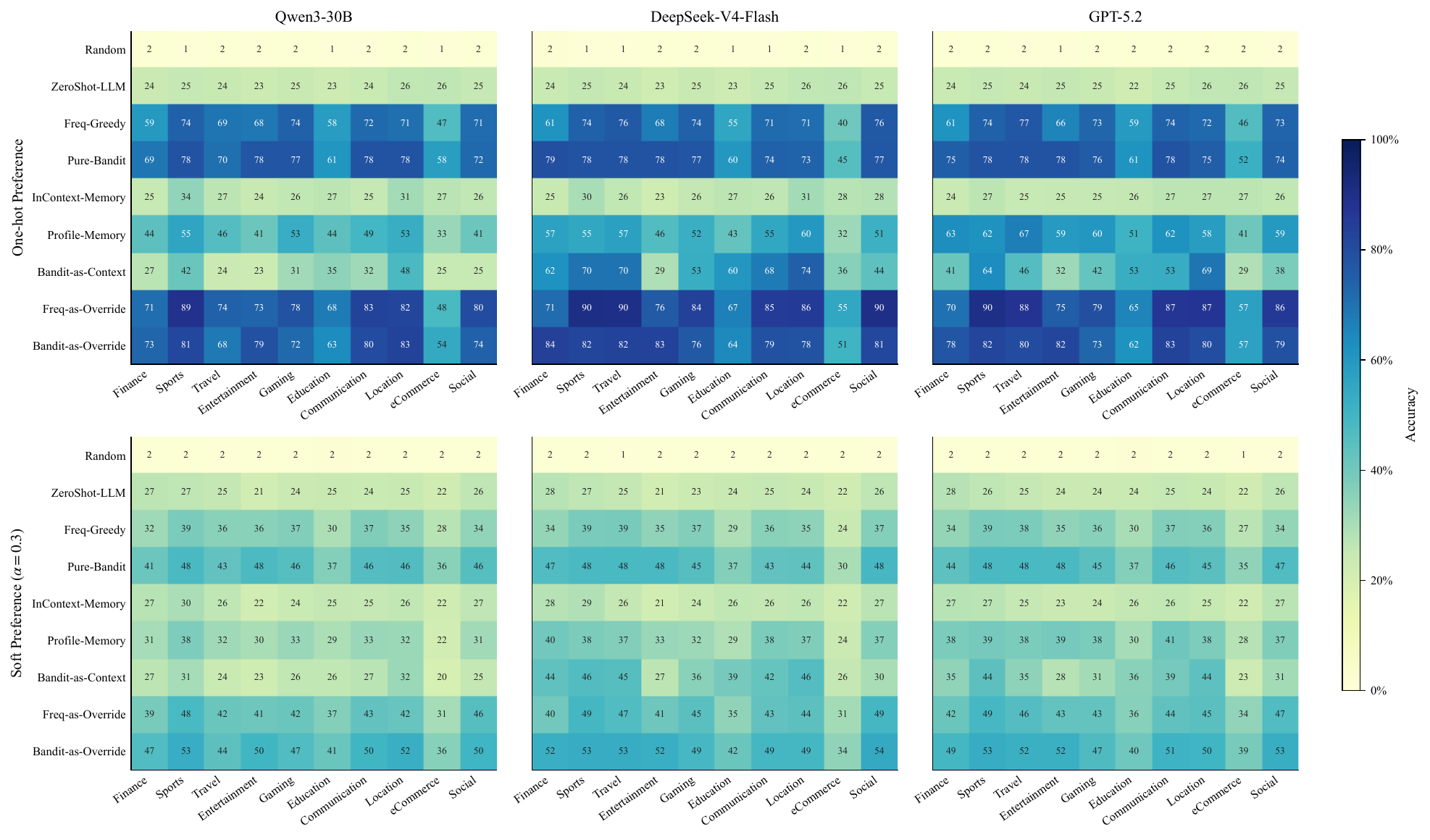}
    \caption{\textbf{Per-domain accuracy heatmaps for the six main experiments.} Rows of the figure correspond to the preference regime (top: one-hot; bottom: soft with $\alpha=0.3$); columns correspond to the LLM backbone. Within each panel, rows are the nine agents and columns are the ten \textsc{ToolBench-60} domains. Cells report the mean training-round accuracy for that $(\text{agent},\text{domain})$ pair, averaged across $N_u=50$ users and $S=3$ seeds, with darker shades denoting higher accuracy (color scale shown on the right). The override-based family (\textsc{Freq-as-Override} and \textsc{Bandit-as-Override}, last two rows of every panel) maintains the highest accuracy uniformly across all ten domains, indicating that the architectural advantage is not concentrated in a particular topical category.
    }
    \label{fig:appendix_per_domain}
    \vspace{-10pt}
\end{figure*}

\Cref{fig:appendix_per_domain} decomposes the accuracy along the ten \textsc{ToolBench-60} domains, displayed as a $9\times 10$ heatmap per experiment. The decomposition is intended to rule out the hypothesis that the proposed architecture's advantage is driven by one or two particularly easy categories. In every one of the six panels, the bottom two rows of the heatmap (\textsc{Freq-as-Override} and \textsc{Bandit-as-Override}) are uniformly the darkest. 
By contrast, \textsc{Profile-Memory} and \textsc{InContext-Memory} exhibit substantial cross-domain variability, with several light cells indicating domains where memory-injection is particularly noisy. 
\section{Preference Recovery for Soft-Setting Results}
\label{sec:more-exp-soft-pref-recovery}

\begin{figure*}[t]
    \centering
    \includegraphics[width=\linewidth]{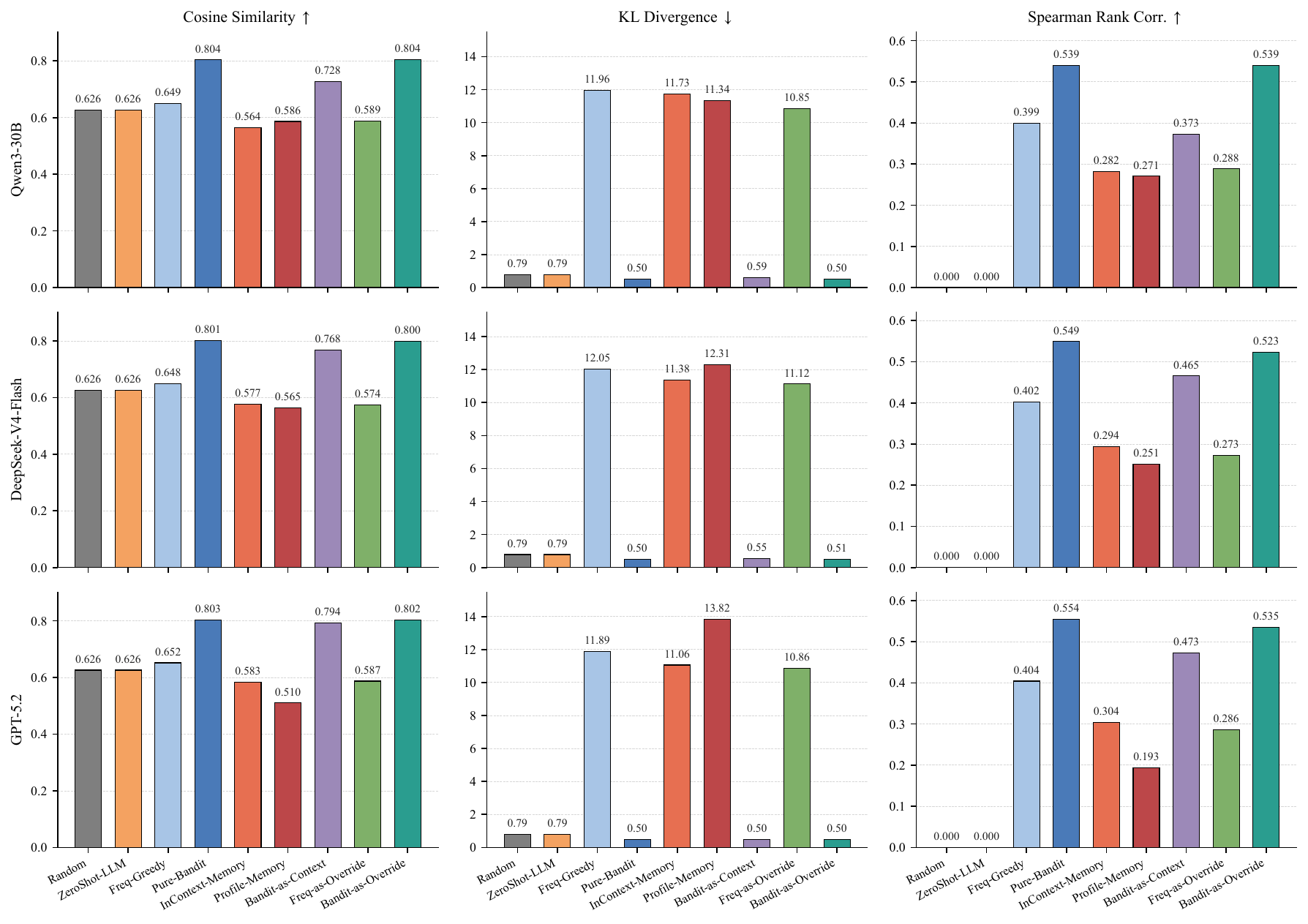}
    \caption{\textbf{Preference-distribution recovery for each agent under the soft preference regime ($\alpha=0.3$).} Rows correspond to the three LLM backbones; columns correspond to the three alignment metrics between each agent's learned per-$(u,d)$ preference $\hat p_{u,d}$ and the user's ground-truth Dirichlet distribution $\pi_u(d)$: cosine similarity, KL divergence, and Spearman rank correlation. Arrows in column headers indicate the direction of improvement. All values are aggregated over $N_u\times|\mathcal{D}|=500$ user--domain pairs per seed across $S=3$ seeds. \textsc{Pure-Bandit} and \textsc{Bandit-as-Override} are indistinguishable on all three metrics across all three backbones, confirming that the LLM override channel of our method does not corrupt the bandit's learned posterior. 
    }
    \label{fig:appendix_preference_recovery}
\end{figure*}

\Cref{fig:appendix_preference_recovery} evaluates how well each agent's \emph{learned} preference distribution $\hat p_{u,d}(k)$ matches the user's ground-truth Dirichlet preference $\pi_u(d)$ under the soft regime, along three complementary metrics. The most informative comparison is between \textsc{Pure-Bandit} and \textsc{Bandit-as-Override}: their bars are indistinguishable on all three metrics and across all three backbones (e.g., Spearman $\approx 0.53$ on every panel), demonstrating that interleaving the LLM override channel \emph{does not corrupt the bandit's posterior}. 
By contrast, \textsc{Freq-as-Override} attains a competitive Spearman score but catastrophic KL divergence, because frequency counting can recover the \emph{mode} of $\pi_u(d)$ but not its full shape. Memory-augmented LLM baselines show the same pathology in even more pronounced form. Taken together, these results support the architectural separation we propose: the local statistical primitive is responsible for posterior estimation, while the LLM is restricted to a narrow override channel that, by construction, leaves the learned posterior intact.

\section{Robustness Stress Tests}
\label{sec:robustness}

The main experiments assume stationary user preferences with immediate, clean binary feedback. This appendix injects the three properties that most sharply separate deployment conditions from that idealization, and measures whether the decoupling advantage survives. All runs use the paper's main protocol (\textsc{ToolBench-60}, Qwen3-30B backbone, soft preferences with $\alpha=0.3$, $N_u=50$ users $\times$ $T=500$ rounds), so every number is directly comparable to the main results.

\paragraph{Corruptions.} We consider three independent modifications of the feedback channel:
\begin{itemize}
    \item \textbf{Noisy feedback.} The observed binary reward is flipped with probability $p = 0.2$. The agent's update therefore consumes a corrupted reward, while the reported regret and accuracy are computed against the true ground truth.
    \item \textbf{Sparse feedback.} Feedback is observed on only $50\%$ of rounds ($\rho = 0.5$), selected independently of the action taken and of its outcome; unobserved rounds skip all agents' state updates.
    \item \textbf{Preference drift.} At round $t = 250$, every user's preference distribution is resampled from a fresh Dirichlet draw, abruptly invalidating all statistics accumulated over the first half of the run. Under this condition the held-out test pool reflects the post-drift preferences.
\end{itemize}

\begin{table}[t]
\centering
\resizebox{\linewidth}{!}{
\begin{tabular}{lcccc}
\hline
\textbf{Agent} & \textbf{Clean} & \textbf{Noise} & \textbf{Sparse} & \textbf{Drift} \\
 & & ($p{=}0.2$) & ($\rho{=}0.5$) & ($t{=}250$) \\ \hline
Random             & 492.1 & 492.1 & 492.1 & 491.5 \\
ZeroShot-LLM       & 373.3 & 375.7 & 375.9 & 376.6 \\
InContext-Memory   & 368.6 & 374.8 & 373.5 & 372.4 \\
Profile-Memory     & 339.3 & 360.3 & 355.4 & 357.2 \\
Bandit-as-Context  & 366.3 & 373.7 & 372.9 & 372.3 \\
Freq-Greedy        & 327.9 & 352.1 & 343.9 & 375.5 \\
Pure-Bandit        & 279.5 & 352.9 & 325.4 & 328.9 \\
Freq-as-Override   & 292.0 & 324.2 & 308.5 & 334.3 \\
\textbf{Bandit-as-Override} & \textbf{262.1} & \textbf{322.0} & \textbf{305.5} & \textbf{303.8} \\ \hline
\end{tabular}
}
\caption{\textbf{Cumulative regret ($\downarrow$) under corrupted feedback.} Qwen3-30B backbone, soft preferences. The ``Clean'' column is the paper's main experiment re-run under the identical protocol and seed. Best per column in bold.}
\label{tab:robustness-regret}
\end{table}

\begin{table}[t]
\centering
\resizebox{\linewidth}{!}{
\begin{tabular}{lcccc}
\hline
\textbf{Agent} & \textbf{Clean} & \textbf{Noise} & \textbf{Sparse} & \textbf{Drift} \\
 & & ($p{=}0.2$) & ($\rho{=}0.5$) & ($t{=}250$) \\ \hline
Random             & 1.5  & 1.5  & 1.5  & 2.4  \\
ZeroShot-LLM       & 23.7 & 23.7 & 23.7 & 24.5 \\
InContext-Memory   & 24.7 & 24.2 & 25.0 & 26.1 \\
Profile-Memory     & 32.6 & 27.9 & 30.9 & 29.4 \\
Bandit-as-Context  & 25.7 & 24.8 & 25.4 & 25.6 \\
Freq-Greedy        & 33.5 & 30.4 & 33.4 & 17.9 \\
Pure-Bandit        & 37.6 & 30.0 & 36.8 & 34.6 \\
Freq-as-Override   & 42.2 & 36.7 & \textbf{42.8} & 29.6 \\
\textbf{Bandit-as-Override} & \textbf{45.6} & \textbf{37.4} & \textbf{42.8} & \textbf{41.6} \\ \hline
\end{tabular}
}
\caption{\textbf{Held-out test accuracy (\%, $\uparrow$) under corrupted feedback.} Under the drift condition the test pool reflects the post-drift preferences. Best per column in bold.}
\label{tab:robustness-acc}
\end{table}

\paragraph{Findings.} \Cref{tab:robustness-regret,tab:robustness-acc} report cumulative regret and held-out test accuracy. Two observations stand out.

\textbf{(i) The advantage does not degrade into a reversal.} \textsc{Bandit-as-Override} attains the lowest regret in all three stress conditions and the highest test accuracy in two of them (under sparsity, \textsc{Freq-as-Override} matches it at the reported precision). Noise and sparsity slow convergence without changing the ranking, and noise hurts more than sparsity, which is the ordering that~\cref{cor:delta} predicts at these corruption levels: flip-noise injects an $O(p)$ asymptotic bias and inflates the rate constant by $1/(1-2p)$, whereas sparsity preserves consistency exactly and inflates the constant only by $1/\sqrt{\rho}$.

\textbf{(ii) Drift separates exploration from memorization.} Both frequency-based agents degrade sharply (\textsc{Freq-Greedy} $33.5 \to 17.9$, \textsc{Freq-as-Override} $42.2 \to 29.6$ test accuracy), because a purely exploitative prior keeps trusting stale statistics. \textsc{Bandit-as-Override} instead re-adapts within the post-drift rounds, recovering $91\%$ of its clean test accuracy ($41.6$ vs.\ $45.6$): the UCB uncertainty bonus automatically re-explores once stale arms begin to fail. 



\section{Semantic-Signal Sweep}
\label{sec:signal-sweep}

The main experiments are run in the regime that defines implicit preference: standard queries are skill-agnostic, so the correct skill is determined entirely by the user's latent preference and not by the query text. This appendix relaxes that property continuously, in order to locate the boundary of the decoupling advantage rather than assert it globally.

\paragraph{Definition.} Let $\sigma$ denote the probability that a standard query carries a semantic hint toward its true skill. At simulation time, each standard query is replaced with probability $\sigma$ by a \emph{hint query} for that round's true (preferred) skill; the held-out test pool is treated identically. Thus $\sigma = 0$ is the paper's main setting, and $\sigma = 1$ means every standard query is semantically resolvable in principle. Explicit queries are unaffected and remain at ratio $\rho_e = 0.10$.

\paragraph{Construction of the hint bank.} For each of the $60$ skills we generated $10$ hint queries ($600$ in total) with an LLM. A hint query describes the target skill's distinctive capability precisely enough that, among the six skills of its domain, only the target is the right match, but it never names the skill. Generated queries were automatically screened against name and brand leakage, i.e.\ any surface form of the target skill's name or its vendor. For example, for the target \textsc{Skyscanner API} in the Travel domain, one generated hint query reads:

\begin{examplebox}{Example hint query ($\sigma > 0$ stimulus)}
I need a travel-data API for a mobile app that can search routes between cities, show fare comparisons by date, and deep-link users to booking partners. What should I use?
\end{examplebox}

\noindent This is resolvable to the right skill from the capability description alone, with no name or brand token present. Hint queries thus occupy an intermediate position between standard queries (no signal, decidable only from accumulated statistics) and explicit queries (the skill name quoted verbatim, decidable from text alone).

\begin{table}[t]
\centering
\resizebox{\linewidth}{!}{
\begin{tabular}{lcccc}
\hline
\textbf{Agent} & $\bm{\sigma{=}0}$ & $\bm{\sigma{=}1/3}$ & $\bm{\sigma{=}2/3}$ & $\bm{\sigma{=}1}$ \\ \hline
Random             & 492.1 & 491.5 & 491.5 & 491.5 \\
ZeroShot-LLM       & 373.3 & 269.9 & 149.4 & \textbf{35.7} \\
InContext-Memory   & 368.6 & 262.0 & \textbf{142.7} & \textbf{30.8} \\
Profile-Memory     & 339.3 & 265.4 & 200.4 & 134.6 \\
Bandit-as-Context  & 366.3 & 288.8 & 197.2 & 113.5 \\
Freq-Greedy        & 327.9 & 321.9 & 328.0 & 327.7 \\
Pure-Bandit        & 279.5 & 271.2 & 256.5 & 240.8 \\
Freq-as-Override   & 292.0 & 284.3 & 283.5 & 282.8 \\
\textbf{Bandit-as-Override} & \textbf{262.1} & \textbf{254.2} & 235.3 & 216.3 \\ \hline
\end{tabular}
}
\caption{\textbf{Cumulative regret ($\downarrow$) as the semantic-signal fraction $\sigma$ varies.} Qwen3-30B backbone, soft preferences. The $\sigma{=}0$ column is the paper's main experiment under the identical protocol and seed. Best per column in bold.}
\label{tab:sweep-regret}
\end{table}

\begin{table}[t]
\centering
\resizebox{\linewidth}{!}{
\begin{tabular}{lcccc}
\hline
\textbf{Agent} & $\bm{\sigma{=}0}$ & $\bm{\sigma{=}1/3}$ & $\bm{\sigma{=}2/3}$ & $\bm{\sigma{=}1}$ \\ \hline
Random             & 1.5  & 2.0  & 2.0  & 2.0  \\
ZeroShot-LLM       & 23.7 & 48.6 & 70.4 & 93.4 \\
InContext-Memory   & 24.7 & 50.6 & \textbf{72.2} & \textbf{93.9} \\
Profile-Memory     & 32.6 & 47.9 & 60.4 & 74.4 \\
Bandit-as-Context  & 25.7 & 45.0 & 60.7 & 77.0 \\
Freq-Greedy        & 33.5 & 35.2 & 33.9 & 33.5 \\
Pure-Bandit        & 37.6 & 43.2 & 47.9 & 56.5 \\
Freq-as-Override   & 42.2 & 42.0 & 42.3 & 42.2 \\
\textbf{Bandit-as-Override} & \textbf{45.6} & \textbf{50.8} & 55.6 & 62.5 \\ \hline
\end{tabular}
}
\caption{\textbf{Held-out test accuracy (\%, $\uparrow$) as the semantic-signal fraction $\sigma$ varies.} Best per column in bold.}
\label{tab:sweep-acc}
\end{table}

\paragraph{Findings.} \Cref{tab:sweep-regret,tab:sweep-acc} sweep $\sigma \in \{0, 1/3, 2/3, 1\}$ under the main protocol with all nine agents present, \textsc{ZeroShot-LLM} serving as the pure-semantic reference. Three points follow.

\textbf{(i) Within the low-signal regime, the advantage is real and not an artifact of zero signal.} At $\sigma = 0$ and $\sigma = 1/3$, \textsc{Bandit-as-Override} is the best agent on both metrics. Moreover, our method is not confined to the zero-signal regime by construction: \textsc{LinUCB}'s context vector already hashes the query text, so the local estimator also benefits as $\sigma$ grows, improving monotonically from $262.1$ to $216.3$ regret. The crossover with in-prompt semantic parsing lies between $\sigma = 1/3$ and $\sigma = 2/3$.

\textbf{(ii) At high signal, the personalization problem itself dissolves, for every method.} At $\sigma = 1$ the best agent, \textsc{InContext-Memory} ($30.8$ / $93.9\%$), is essentially indistinguishable from \textsc{ZeroShot-LLM} ($35.7$ / $93.4\%$), which carries no memory, no prior, and no personalization machinery at all. The agents that inject preference statistics into the prompt do markedly worse (\textsc{Profile-Memory} $74.4\%$, \textsc{Bandit-as-Context} $77.0\%$): when the query already says what the user wants, accumulated statistics can only distract from it. In this regime nothing needs to be learned about the user, plain LLM execution is the right and simplest design, and no preference mechanism should be deployed.

\textbf{(iii) The scope of the paper's claim.} The high-$\sigma$ region is therefore not a regime where the method is beaten at its own task, but one where there is no personalization task left. Once a query semantically identifies its skill, the preference is no longer implicit. Accordingly we state the conclusion of~\cref{ssec:main-results} conditionally: the decoupling advantage holds in the low-signal regime ($\sigma \lesssim 1/3$ in this sandbox), which is the implicit-preference setting the paper defines as its subject, while plain in-prompt semantic parsing suffices once most queries carry recoverable signal.

\section{Why In-Context Statistical Priors Fail}
\label{sec:why-incontext-fails}

\textsc{Bandit-as-Context} receives exactly the same statistical prior as \textsc{Bandit-as-Override}, yet performs far worse. A natural hypothesis is that this is an artifact of \emph{how} the prior is rendered in the prompt, in which case a better representation or a firmer instruction should close the gap. This appendix tests that hypothesis along the representation axis and the model-capability axis, and localizes the failure with a per-round deviation log.

\paragraph{Model-capability axis.} From the existing main results, \textsc{Bandit-as-Context}'s regret in the soft-preference regime is $369.6$ (Qwen3-30B), $309.6$ (DeepSeek-V4-Flash), and $327.1$ (GPT-5.2), all far above \textsc{Bandit-as-Override} ($264.8$ / $255.9$ / $256.8$). Even the strongest backbone does not use the in-context prior well, so limited model capability alone does not explain the failure.

\paragraph{Representation and instruction-strength axis.} We designed three additional prompt formulations of the identical prior, varying both the representation of the statistics and the strength of the decision instruction, and ran all four under the main protocol and seed:
\begin{enumerate}
    \item \textbf{Percentages} (the paper's form). Each skill is listed with its learned preference share, e.g.\ ``\texttt{ToolX: 62\%}'', followed by the soft instruction ``\textit{If the query explicitly names a tool, use it. Otherwise, prefer the tool with the highest user preference.}''
    \item \textbf{Raw counts.} The percentages are replaced by the underlying evidence, e.g.\ ``\texttt{ToolX: 8/10 successful (80\%)}'', testing whether the model uses raw counts better than normalized shares.
    \item \textbf{Verbal grades.} The same prior is translated into natural-language tiers (``strongly preferred'' / ``moderately preferred'' / ``occasionally used'' / ``rarely used''), testing whether the model follows linguistic descriptions better than numbers.
    \item \textbf{Instruction-hardened.} The percentages are kept, but the soft instruction is replaced by an explicit command that writes the override rule verbatim into the prompt.
\end{enumerate}

\begin{promptbox}{Instruction-hardened variant}
Select the best tool for this query.\\
Available tools (with learned user preference \%):\\
\{tools\_block\}

Query: ``\{query\}''

RULE (follow it exactly, no exceptions): if the query explicitly
names one of the tools above, select that tool. Otherwise you MUST
select the tool with the highest user preference percentage.\\
Reply with only JSON: \{``tool'': ``<tool name>''\}
\end{promptbox}

\paragraph{Deviation rate.} Each variant additionally logs, per round, the prior's argmax against the LLM's actual choice. We define the \emph{deviation rate} as the fraction of standard queries on which the two differ. Because the prior's argmax is precisely what \textsc{Bandit-as-Override} executes by construction, this quantity isolates the execution gap between the two architectures while holding the statistics fixed.

\begin{table}[t]
\centering
\resizebox{\linewidth}{!}{
\begin{tabular}{lcccc}
\hline
\textbf{Prior representation} & \textbf{Regret} & \textbf{Test} & \textbf{Explicit} & \textbf{Dev.} \\
 & $\downarrow$ & \textbf{acc.} $\uparrow$ & \textbf{acc.} $\uparrow$ & \textbf{rate} \\ \hline
Percentages (paper's form)  & 366.1 & 25.9 & 96.4 & 86.7\% \\
Raw counts                  & 349.6 & 29.9 & 91.2 & 61.0\% \\
Verbal grades               & 365.5 & 26.8 & 96.4 & 84.4\% \\
Instruction-hardened        & 339.5 & 31.0 & 89.2 & 68.8\% \\ \hline
\textbf{Bandit-as-Override} & \textbf{262.1} & \textbf{45.6} & 94.4 & \textbf{0\%} \\ \hline
\end{tabular}
}
\caption{\textbf{Four prompt formulations of the same statistical prior.} Qwen3-30B, soft preferences, same protocol and seed as the main experiment. ``Dev.\ rate'' is the fraction of standard queries on which the LLM's choice differs from the prior's argmax; it is $0\%$ by construction for the override architecture, which executes the argmax in an external module.}
\label{tab:incontext-variants}
\end{table}

\paragraph{Findings.} \Cref{tab:incontext-variants} reports the four variants. The decisive column is the last one.

\textbf{(i) Every formulation deviates far too often.} Across all four representations and instruction strengths, the LLM departs from the prior's argmax on $61$--$87\%$ of standard queries, even though following the prior is roughly $3\times$ more accurate than deviating ($41.2$--$46.2\%$ accuracy when following, versus $11.6$--$16.8\%$ when deviating). However the prior is phrased, a probabilistic generator does not reliably execute the deterministic rule ``pick the argmax''.

\textbf{(ii) This is why no representation closes the gap, and why an external module does.} Even the instruction-hardened variant, which represents the ceiling of what an in-context method can be pushed to, still deviates on $68.8\%$ of standard queries and reaches only $339.5$ / $31.0\%$, far from \textsc{Bandit-as-Override}'s $262.1$ / $45.6\%$. Its firmer obedience moreover costs explicit-query accuracy ($89.2\%$ versus $96.4\%$), so further hardening is not a free direction: pushing the model to obey the statistical rule erodes precisely the semantic flexibility that the override channel is meant to preserve. Notably, even \textsc{Pure-Bandit} ($279.5$ / $37.6\%$), the local estimator with no LLM at all, outperforms every in-context variant.

\paragraph{Conclusion.} Across representations and instruction strengths, an in-context prior narrows but never closes the gap. Executing a deterministic decision rule through a probabilistic generator incurs an irreducible execution noise, which the structured override channel eliminates architecturally rather than statistically: the argmax is executed by the local module deterministically, and the LLM is consulted only for the semantic override decision. Combined with the sweep of~\cref{sec:signal-sweep}, which locates the crossover on the signal axis between $\sigma = 1/3$ and $\sigma = 2/3$, this characterizes both axes along which the in-prompt mechanism can catch up: more semantic signal in the query, but not a better rendering of the statistics.


\section{Hardened Memory and Persona Baselines}
\label{sec:hardened-memory}

The comparison against memory-augmented agents could in principle be sensitive to prompt design and context budget: perhaps the gap reflects a weak memory-prompt implementation rather than the decision architecture. \textsc{Profile-Memory} was already built to preclude this reading, receiving the exact per-$(u,d,k)$ (successes, attempts) counts that \textsc{LinUCB} consumes, in a lossless structured format modeled on production memory systems (\cref{ssec:hparams}). This appendix pushes the memory family further with three stronger baselines.

\paragraph{Baselines.}
\begin{itemize}
    \item \textbf{ACE-style}~\citep{zhang2026agentic}: maintains a per-$(u,d)$ structured ``playbook'' that is incrementally revised every $25$ rounds by an LLM reflection call over recent (query, action, reward) triples, following the grow-and-refine pattern, and injected at selection time.
    \item \textbf{SynthesizeMe-style}~\citep{ryan2025synthesizeme}: every $50$ rounds, an LLM synthesizes a natural-language persona from the user's interaction history (e.g.\ ``\textit{this user prefers X for shopping-data queries}''), which is prepended to the selection prompt.
    \item \textbf{Profile-Memory+}: identical statistics to \textsc{Profile-Memory}, but with a carefully engineered prompt comprising an explicit decision rule (``pick the highest-success-rate tool unless the query names one''), few-shot examples, and a required intermediate ranking step.
\end{itemize}
All three were run on Qwen3-30B, the paper's main backbone, under the identical protocol and seed.

\begin{table}[t]
\centering
\resizebox{\linewidth}{!}{
\begin{tabular}{lcc}
\hline
\textbf{Agent} & \textbf{Regret} $\downarrow$ & \textbf{Test acc.\ (\%)} $\uparrow$ \\ \hline
ZeroShot-LLM (paper)        & 373.3 & 23.7 \\
Profile-Memory (paper)      & 339.3 & 32.6 \\
SynthesizeMe-style (new)    & 376.2 & 23.5 \\
ACE-style (new)             & 374.4 & 23.1 \\
Profile-Memory+ (new)       & 311.8 & 37.2 \\ \hline
\textbf{Bandit-as-Override} & \textbf{262.1} & \textbf{45.6} \\ \hline
\end{tabular}
}
\caption{\textbf{Hardened memory and persona baselines.} Qwen3-30B, soft preferences, identical protocol and seed. Rows marked ``paper'' are the main experiment re-run under that seed.}
\label{tab:hardened-memory}
\end{table}

\paragraph{Findings.} \Cref{tab:hardened-memory} reports the comparison.

\textbf{(i) Prompt engineering helps, but does not close the gap.} \textsc{Profile-Memory+} improves markedly over \textsc{Profile-Memory} (regret $339.3 \to 311.8$, accuracy $32.6\% \to 37.2\%$). Yet with identical statistics and the decision rule explicitly engineered into its prompt, it still trails \textsc{Bandit-as-Override} by roughly $50$ regret points and $8.4$ accuracy points. The portion of the gap recoverable by prompt design is therefore bounded well short of the override architecture, and the residual is attributable to the decision decoupling itself, consistent with the deviation-rate analysis of~\cref{sec:why-incontext-fails}.

\textbf{(ii) Natural-language compression loses the signal the decision needs.} The two reflection and persona baselines land at \textsc{ZeroShot-LLM} level ($376.2$ / $23.5\%$ and $374.4$ / $23.1\%$ versus $373.3$ / $23.7\%$), despite consuming the same interaction history. A persona or playbook is a lossy point summary, and what the compression discards is exactly the calibrated per-skill distribution that skill selection requires. We stress that this is not a deficiency of these methods: they target semantic and stylistic preference, where a persona is the right representation. 


\section{Query Templates}
\label{sec:query-templates}
\renewtcolorbox{promptbox}[1]{%
    breakable,
    enhanced,
    colback=gray!4,
    colframe=gray!55!black,
    boxrule=0.45pt,
    arc=2.5pt,
    left=6pt, right=6pt, top=5pt, bottom=5pt,
    title=\textsf{\small\bfseries #1},
    fonttitle=\sffamily\small\bfseries,
    coltitle=white,
    colbacktitle=gray!55!black,
    attach boxed title to top left={yshift=-2mm, xshift=4mm},
    boxed title style={colback=gray!55!black, sharp corners, boxrule=0pt},
    fontupper=\small\ttfamily,
    before upper={\setlength{\parindent}{0pt}\setlength{\parskip}{2pt}},
}

\renewtcolorbox{examplebox}[1]{%
    breakable,
    enhanced,
    colback=blue!2,
    colframe=blue!35!black,
    boxrule=0.45pt,
    arc=2.5pt,
    left=6pt, right=6pt, top=5pt, bottom=5pt,
    title=\textsf{\small\bfseries #1},
    fonttitle=\sffamily\small\bfseries,
    coltitle=white,
    colbacktitle=blue!35!black,
    attach boxed title to top left={yshift=-2mm, xshift=4mm},
    boxed title style={colback=blue!35!black, sharp corners, boxrule=0pt},
    fontupper=\small,
    before upper={\setlength{\parindent}{0pt}\setlength{\parskip}{2pt}},
}

This appendix documents the textual prompts used throughout our benchmark.
We organize them into three parts: (i)~the prompts we feed to an LLM to generate the query template banks, before any experiment is run; (ii)~representative examples of the resulting \emph{standard} and \emph{explicit} queries on a subset of domains; and (iii)~the prompts that each agent issues to a language model at inference time during the online evaluation.

\subsection{Query Template Generation}
\label{app:prompts:generation}

We construct the benchmark's query bank in a single offline pass using \textsc{GPT-5.2}~\citep{gpt52}. For each of the ten \textsc{ToolBench-60} domains, we issue two distinct prompts. The \emph{standard} prompt produces queries that describe a user's information need in natural language but \emph{deliberately avoid naming any specific tool (skill)}; such queries can be answered correctly only if the agent has inferred the user's underlying preference. The \emph{explicit} prompt produces queries that explicitly name one particular tool (skill) from the domain's six options, so that a competent language model can recover the target tool from the text alone. We generate 20 standard templates per domain and 1-5 explicit templates per tool; the exact counts and the resulting raw text are released alongside the code.

\begin{promptbox}{Standard query generation prompt}
Generate \{n\_per\_domain\} diverse user queries that someone might ask
when needing a tool/API from the ``\{domain\}'' category.

Tools in this category:\\
\hspace*{1em}- \{tool\_1\}: \{description\_1\}\\
\hspace*{1em}- \{tool\_2\}: \{description\_2\}\\
\hspace*{1em}\ldots

IMPORTANT RULES:\\
\hspace*{1em}- Do NOT mention any specific tool name in the queries\\
\hspace*{1em}- Each query should be a natural user request (1-2 sentences)\\
\hspace*{1em}- Cover diverse use cases within this category\\
\hspace*{1em}- Write in English

Return a numbered list (1. query, 2. query, \ldots).
\end{promptbox}

\begin{promptbox}{Explicit query generation prompt}
Generate \{n\_per\_tool\} user queries that explicitly request using
``\{tool\}'' (a \{domain\} tool: \{description\}).

IMPORTANT RULES:\\
\hspace*{1em}- Each query MUST mention ``\{tool\}'' by name\\
\hspace*{1em}- The query should be a natural user request (1-2 sentences)\\
\hspace*{1em}- Write in English

Return a numbered list.
\end{promptbox}

The dual constraint ``do not mention any tool'' for standard queries versus ``must mention this tool by name'' for explicit queries creates a clean binary partition of the benchmark's stimuli along the axis our method is designed to handle: standard queries are decidable only via \emph{personalized} statistical inference, while explicit queries are decidable from the text alone. The standard-versus-explicit dichotomy is the structural property that the two-bar test-accuracy decomposition in the main paper makes visible. 

\subsection{Example Queries by Domain}
\label{app:prompts:examples}

To give the reader a concrete sense of the linguistic style and difficulty of the synthetic stimuli, we list three illustrative domains from \textsc{ToolBench-60} below. For each domain we show three standard queries (which require user-preference inference) and three explicit queries that each name a distinct tool (which require only text-level recognition). The full banks of standard and explicit templates across all ten domains are released with the benchmark configuration.

\begin{examplebox}{Finance domain}
\textbf{Standard queries:}
\begin{itemize}\setlength{\itemsep}{1pt}
    \item I need end-of-day data for major indices and ETFs, plus a way to backfill missing market days automatically. 
    \item Give me daily and historical metrics for thousands of cryptocurrencies, including market cap and volume, with simple pagination and caching support. 
    \item Pull real-time and 10-year historical daily prices for a list of US stocks, and also include splits and dividends where available.
\end{itemize}

\smallskip
\textbf{Explicit queries:}
\begin{itemize}\setlength{\itemsep}{1pt}
    \item Can you run an \emph{``Extended ACH and Check Prescreen''} on this routing and account number to confirm the account is open and in good standing before we accept an ACH payment?
    \item Use \emph{Twelve Data} to pull real-time and 1-minute intraday prices for AAPL for today, and return the latest quote plus OHLCV
    data.
    \item I want to integrate the \emph{ChangeNOW crypto exchange} into my checkout so customers can buy and exchange crypto; can you outline the steps and what I need to get started?
\end{itemize}
\end{examplebox}

\begin{examplebox}{Travel domain}
\textbf{Standard queries:}
\begin{itemize}\setlength{\itemsep}{1pt}
    \item I need to search and compare round-trip flight options from NYC to Tokyo for flexible dates next month, including price trends and the cheapest day to fly.
    \item Can you find the lowest fares for a multi-city itinerary (London $\to$ Rome $\to$ Athens $\to$ London) and return results with cabin class, baggage rules, and cancellation terms?
    \item I want to monitor a specific route (SFO to LAS) and alert me when the price drops below a threshold within the next 60 days.
\end{itemize}

\smallskip
\textbf{Explicit queries:}
\begin{itemize}\setlength{\itemsep}{1pt}
    \item Use the \emph{Skyscanner API} to find the cheapest round-trip flights from New York (JFK) to London (LHR) for flexible dates in the next two months, and return the top 10 options.
    \item Use \emph{Airbnb listings} to pull all active listings in the Lisbon market, including nightly price, cleaning fee, minimum stay, and availability for the next 60 days.
    \item Use \emph{webcams.travel} to find nearby landscape webcams around my current location, and show the top 10 with a quick
    preview and distance for each.
\end{itemize}
\end{examplebox}

\begin{examplebox}{Social domain}
\textbf{Standard queries:}
\begin{itemize}\setlength{\itemsep}{1pt}
    \item I need to fetch a user’s public profile info and most recent posts, then compute engagement metrics and growth over time. 
    \item Can you help me set up phone-number verification via SMS with a single API call, and automatically expire codes after five minutes?
    \item I want to automate posting to my social account on a schedule, rotate proxies, and avoid triggering anti-bot blocks. 
\end{itemize}

\smallskip
\textbf{Explicit queries:}
\begin{itemize}\setlength{\itemsep}{1pt}
    \item Use \emph{Twitter AIO} to monitor tweets mentioning my brand name in real time and summarize the key themes every hour. Include links, author handles, and engagement stats for the top posts. 
    \item Can you integrate \emph{Branch Metrics} in our iOS and Android apps to track sign-up and purchase events and tie them back to the original referral source?
    \item Can you use \emph{TikTok Video No Watermark\_v2} to download this TikTok and give me the direct no-watermark video link: https://www.tiktok.com/@user/video/1234567890? 
\end{itemize}
\end{examplebox}

A few features are worth pointing out. Standard queries are \emph{intentionally skill-agnostic}: a competent agent that has not learned the user's preference cannot do meaningfully better than uniform random among the six in-domain tools. Explicit queries, by contrast, quote the target skill name (in our shown examples, we render the quoted phrase in italics for readability), so a language model that can read the query can identify the target skill without any history. The two query types together cover the two sub-problems: personalization and explicit override, which motivate our architectural split. 

\subsection{Inference-Time Prompts}
\label{app:prompts:inference}

At inference time, every agent that involves an LLM call issues one or two structured prompts per round, each constrained to return a strict JSON object. We list below the three prompts that drive the comparisons reported in the main paper: the shared domain classifier (used by \emph{all} agents that require an inferred domain, including \textsc{Pure-Bandit}), the binary override probe used by \textsc{Bandit-as-Override} and \textsc{Freq-as-Override}, and the chain-of-thought selection prompt used by \textsc{Bandit-as-Context}. For brevity we omit the prompts of the LLM-only baselines (\textsc{ZeroShot-LLM}, \textsc{InContext-Memory}, \textsc{Profile-Memory}), which differ only in what user-history context they prepend to a similar selection instruction.

\paragraph{Shared domain classifier.}
This prompt is issued exactly once per query and its inferred-domain output is broadcast to every domain-aware agent. The placeholder \texttt{\{domains\_block\}} expands to a multi-line listing of the ten domains together with the first four tool names of each. Using a shared classifier guarantees that all agents operate over the same domain signal, removing an otherwise unfair advantage that statistical agents would receive from ground-truth domain access.

\begin{promptbox}{Shared domain classifier}
Which domain does this query belong to?

Domains and their tools:\\
\{domains\_block\}

Query: ``\{query\}''\\
Reply with only JSON: \{``domain'': ``<exact domain name>''\}
\end{promptbox}

\paragraph{Override probe} (\textsc{Bandit-as-Override} and \textsc{Freq-as-Override}). 
The defining design choice of \textsc{Bandit-as-Override} agent and \textsc{Freq-as-Override} is to ask the LLM a \emph{single binary question}: does the query explicitly name a tool? The probe deliberately omits tool descriptions and includes only the tool \emph{names} for the agent's inferred domain, so the LLM's task reduces to lexical recognition rather than open-ended selection. This narrow interface is what allows the LLM to act as an exception handler without contaminating the bandit's default selection on the $\sim$90\% of standard queries.

\begin{promptbox}{Override probe}
Does this query explicitly request a specific tool by name?\\
Available tools: \{tool\_1\}, \{tool\_2\}, \ldots, \{tool\_K\}\\
Query: ``\{query\}''

If the query mentions a specific tool name, reply:\\
\hspace*{1em}\{``override'': true, ``tool'': ``<tool name>''\}\\
If not, reply:\\
\hspace*{1em}\{``override'': false\}
\end{promptbox}

\paragraph{Selection-with-prior} (\textsc{Bandit-as-Context}). 
For comparison, the \textsc{Bandit-as-Context} baseline exposes the bandit's posterior to the LLM as a tempered-softmax prior over the domain's six tools, and asks the LLM to perform the \emph{full selection} itself rather than answering an override question. The placeholder \texttt{\{tools\_block\}} expands to lines of the form ``\texttt{- \{tool\_name\}: 38\%}'' sorted by descending prior. 

\begin{promptbox}{Selection-with-prior prompt}
Select the best tool for this query.\\
Available tools (with learned user preference \%):\\
\{tools\_block\}

Query: ``\{query\}''

If the query explicitly names a tool, use it. Otherwise, prefer the
tool with the highest user preference.\\
Reply with only JSON: \{``tool'': ``<tool name>''\}
\end{promptbox}

In all three prompts we set the decoding temperature to $0.0$ and parse the response as JSON. On parse failure (empty body, malformed JSON, or a tool name that does not appear in the domain's tool list), the client first attempts a regular-expression and substring match against the candidate tools, and only on total failure falls back to the default action (the first listed domain for the classifier, the bandit's selection for the override probe, and the first tool for \textsc{Bandit-as-Context}). Empirically the fallback path fires on $\lesssim 0.5\%$ of calls with the backbones reported in \cref{ssec:hardware}.

\section{License of Scientific Artifacts}
Our constructed simulation sandbox, \textsc{ToolBench-60}, is derived from the original \textsc{ToolBench}~\citep{qin2023toolllm} dataset. We acknowledge the creators of ToolBench, which is distributed under the Apache-2.0 License. The newly curated \textsc{ToolBench-60} benchmark is open-sourced and released under the MIT License to facilitate further research in personalized agent interactions. 
\section{Use of AI Assistant}
During the preparation of this work, we utilized AI assistants to enhance productivity in both software development and manuscript preparation. Specifically, Claude Code was employed to assist in drafting, structuring, and debugging the underlying experimental framework and execution pipeline. Additionally, Google's Gemini was utilized to help refine, draft, and polish the textual content of this manuscript. In all cases, we carefully reviewed, verified, and take full responsibility for the accuracy and originality of the submitted work. 

\end{document}